%% file: iclr2026_conference.tex
\documentclass{article} % For LaTeX2e
\usepackage{iclr2026_conference,times}

\input{math_commands.tex}

\usepackage{hyperref}
\usepackage{url}
\usepackage{subcaption}
\usepackage{makecell}
\title{Architecture-driven Shift: towards a lightweight selector for capturing the trends of logit shift}

\author{Zhong Ye \& Yu Hu\thanks{Corresponding author} \\
School of Computer Science and Technology\\
Guangdong University of Technology\\
Guangzhou, China \\
\texttt{yezhong@mails.gdut.edu.cn, huyu@gdut.edu.cn} \\
\And
Ruilin Tang \\
School of Computer Science and Engineering \\
South China University of Technology \\
Guangzhou, China \\
\texttt{t1344409@gmail.com} \\
}

\usepackage{booktabs}
\usepackage{multirow}
\usepackage{amsmath}
\usepackage{amssymb}
\usepackage{mathtools}
\usepackage{amsthm}
\theoremstyle{plain}
\newtheorem{theorem}{Theorem}[section]
\newtheorem{proposition}[theorem]{Proposition}
\newtheorem{lemma}[theorem]{Lemma}

\theoremstyle{definition}

\newtheorem{assumption}[theorem]{Assumption}
\theoremstyle{remark}
\newtheorem{remark}[theorem]{Remark}

\iclrfinalcopy % Uncomment for camera-ready version, but NOT for submission.
\begin{document}

\maketitle

\begin{abstract}
Continual Learning (CL) is a practical paradigm to utilize power of deep pre-trained neural networks, but which pre-trained model has a better ability to balance ``Plasticity-Stability", deserving to be chosen? The logit shift serves as a natural proxy because it represents the logit shift in CL scenarios. However, obtaining the logit shift requires huge computational cost, which hinders large-scale model selection. Existing theoretical analyses fail to offer an efficient alternative because of the assumption of uniform hidden layer widths, which ignores the structural heterogeneity (variable width and depth) of real-world architectures. This raises a critical question: \textbf{what theoretically relationship can be identified between heterogeneous architecture and logit shift on prior tasks (that the model has been trained on)?} To answer the question, we decouple logit shift into architecture dependency and data dependency to establish our framework, which reveals that \textbf{the combination of two dependency, defined as Architecture-driven Shift (ADS), that can capture the logit shift tendency well computable with few data samples}. Specifically, for a well-optimized model on prior tasks, higher ADS is associated with a larger logit shift after training on the current task, which derived based on three mechanistic components: (1) spectral norm scaling of weight matrix gradients with layer width, (2) the optimization path length of the new task, and (3) the asymptotic task conflict in wide networks. Extensive empirical results across more than 175 diverse architectures demonstrate a strong monotonic correlation (the weakest Spearman's $r_s=0.731$) between ADS and logit shift. Practically, we demonstrate that ADS can serve as a lightweight proxy of the expected calibration error, which is a widely used metric for reliable CL model selection, on three datasets across six scenarios.

% \textblue{Furthermore, we demonstrate that ADS can serve as a lightweight proxy for Expected Calibration Error (ECE), offering a pratical coarse-grained selector in reliable CL model selection.}

% \textblue{Crucially, ADS acts as a effective coarse-grained filter: the lowest 5–10\% reliably identifies the most stable architectures, while the highest 20\% (likely miscalibrated on prior tasks) can be safely pruned.} 

% [Main Point/My Claim] In order to figure out the above problem, we propose a theoretical framework, which \textbf{first derive architecture-driven shift (ADS) based on heterogeneous architecture, and reveal that architecture provides a structural prior that captures the logit shift tendency}: for trainable model who well optimized on prior tasks, higher ADS is associated with a larger logit shift after training on the current task. 

\end{abstract}

\section{Introduction} \label{sec:introduction}
With the growing capacity of pretrained deep neural networks (DNNs), Continual Learning (CL) is a practical paradigm that requires a pre-trained model sequentially adapt to a stream of new tasks (Plasticity) while maintaining performance on the original task as much as possible (Stability)~\citep{li2017learning, wu2019large}. For different applications, we can select a more suitable model from several pre-trained models according to the ``Plasticity-Stability" trade-off. For instance, in scenarios where stability is prioritized, we select the model that has a higher stability. Logit shift is a natural and general proxy for quantifying ``Plasticity-Stability" trade-off in CL~\citep{benjamin2018measuring, guha2024diminishing}, yet its computation is expensive because it requires a full training process on new tasks, forcing researchers to rely on CL metrics (e.g. average accuracy and forgetting) instead. 

Since CL metrics based on output are computed directly from the model's logits, any architectural influence observed on these metrics can originate from an underlying architectural dependency within the logit shift itself. Prior studies have observed correlations between architectural variables and CL metrics, suggesting that that \textbf{logit shift contains an inherent architecture dependency that may be computed based on the architectural variables alone, enabling us to bypass the full training process.} Specifically, enlarging the width of hidden layers~\citep{mirzadeh2022wide} or overall model scale~\citep{ramasesh2021effect} serves as a simple method for improving learning accuracy and mitigating catastrophic forgetting\footnote{Catastrophic Forgetting refers to the phenomenon where models sequentially adapt to new tasks, requiring parameter updates, which inadvertently disrupt the representations of previous tasks in CL.~\citep{mccloskey1989catastrophic}}. Yet these studies remain largely at the level of the CL metric, which is merely downstream consequences. To gain a more fundamental mechanistic understanding, it is crucial to analyze logit shift directly. Although there has been growing interest in modeling the relationship between architectures and logit shift~\citep{guha2024diminishing, kim2025measuring}, these studies rely on the assumption of a homogeneous architecture (uniform width of hidden layers), which is insufficient to provide practical guidance for real-world model selection. This motivates us to establish a more comprehensive theoretical analysis to cover heterogeneous architecture. 

% i.e., to answer the question: given a model trained sequentially on several new tasks, what relationship between heterogeneous architecture (variable widths and depths) and logit shift on prior tasks can be theoretically identified?
% To answer the problem above, 

In this paper, we propose a theoretical framework that explicitly models the relationship between the heterogeneous architecture of the typical fully-connected neural network (FNN) and logit shift, which we define as Architecture-driven Shift (ADS). As mentioned above, logit shift contains an inherent architecture dependency that is directly determined by models' width and depth, which reduce computational resource costs. Therefore, we conceptually decouple logit shift into two components: (1) The architecture dependency and (2) the data dependency that is efficiently calibrated through a subset of the dataset. 

To theoretically construct ADS, our analysis proceeds in both static structural properties and dynamic optimization behaviors. On the static structural level, we formulate the mathematical link between the width of hidden layers and the spectral norm of the weight gradients (Proposition~\ref{prop:gradient_scaling_for_weight_matrix}). Specifically, we characterize how the spectral norm inherently scales with layer width, driven by the varying energy of forward activations. Concurrently, on the dynamic optimization level, we decouple the adaptation process into macro- and microscopic phenomena. Macroscopically, we establish how the width of hidden layers constrains the scale of the optimization trajectory (such as parameter displacement) during new task learning (Lemma~\ref{lemma:pathlen_scaling_law}). Microscopically, we reveal an interesting property of task interference: as networks grow wider, random high-dimensional noise asymptotically diminishes, purifying the cross-task conflicts into a deterministic rule governed by semantic task representations rather than random gradient conflicts (\eqref{eq:theorem_depth_term}).

By unifying static structural properties and dynamic optimization behaviors, we derive ADS, which can be efficiently computed using only a few calibration samples, enabling architectural variables to capture the logit shift tendency and paving the way for large-scale CL model selection. Specifically, for reliable CL model selection, we show that ADS inherently correlates with Expected Calibration Error (ECE), add a intuitive description, enabling it to serve as a lightweight coarse-grained selector without full training, add reference.

The main contributions can be summarized as follows:
(1) \textbf{Theoretical Framework for Heterogeneous Architecture of FNN.} We present the first theoretical framework to analyze a finer-grained object, i.e., logit shift, instead of the output-based metrics in CL with realistic, heterogeneous architectures (variable widths and depths). This bridges a critical gap in CL theory, as prior analyses were restricted to simplified, uniform-width models and thus could not provide practical guidance for real-world model selection. (2) \textbf{Mechanistic Decomposition of logit shift.} Our framework provides a novel mechanistic decomposition of the logit shift. We demonstrate that this complex phenomenon is governed by the synthesis of three distinct analytical components: a static scaling law of the weight gradients (Proposition~\ref{prop:gradient_scaling_for_weight_matrix}), the macroscopic dynamics of the optimization trajectory (Lemma~\ref{lemma:pathlen_scaling_law}), and the microscopic purification of task semantic interference in wide networks (\eqref{eq:theorem_depth_term}). This decomposition offers an unprecedented level of mechanistic insight into how architecture shapes logit-space dynamics. (3) \textbf{Empirical Validation and Application.} We provide comprehensive empirical evidence confirming the predictions of our theoretical framework across diverse architectures of FNN and Transformer. Through comprehensive experiments across a diverse suite of tasks (ranging from easy datasets like Split/Rotated MNIST to more complex datasets such as CIFAR-100 and ImageNet subsets) and both transfer and continual learning paradigms, we demonstrate that the ADS maintains a strong monotonic correspondence with the empirically observed logit shift. Crucially, we validate ADS as a lightweight selector for ECE, which is a well adopted metric for reliable CL model selection.

\section{Related Works}
% Given that the unified manifestation of parameter updates and representational shifts is the model's \textbf{logit shift},

\subsection{Methodological Approaches and Computational Bottlenecks in CL}
Continual Learning (CL) aims to enable deep neural networks to sequentially adapt to a stream of new tasks. A fundamental challenge in this paradigm is catastrophic forgetting\footnote{First discussed by \citet{mccloskey1989catastrophic}.}, where updating parameters for current tasks disrupts previously learned representations. Extensive research has concentrated on addressing this problem from three main perspectives: regularization-based, replay-based, and optimization-based methods.

Elastic Weight Consolidation (EWC)~\citep{kirkpatrick2017overcoming}, one of the classic regularization-based methods, estimates parameter importances by calculating the diagonal of the Fisher Information Matrix after learning a task. It mitigates forgetting by penalizing changes to these critical parameters during subsequent sequential training. 
In contrast, replay-based approaches, such as Hindsight Experience Replay (HER)~\citep{aygun2022proving}, focus on the data-level. They store a small subset of samples from previous tasks in an episodic memory and interleave them with current task data during training to preserve learned knowledge. 
Rather than manipulating parameters or datasets, Orthogonal Gradient Descent (OGD)~\citep{farajtabar2020orthogonal} directly intervenes in the optimization space. OGD projects the parameter gradients of new tasks onto the orthogonal complement of the subspace spanned by the gradients of prior tasks, mathematically ensuring that the predictions on earlier tasks remain invariant under gradient updates.

While these methodologies exhibit strong empirical performance, they all share a common limitation: they rely heavily on computationally expensive active training interventions (such as calculating high-dimensional Fisher matrices, managing memory buffers, or computing gradient projection matrices during training). This hinders large-scale, pre-training stage model selection. 

\subsection{Architectural Influence and Training Dynamics in CL}
% To understand how the structural properties of deep networks intrinsically affect logit shift, recent research has shifted its perspective from algorithms to architecture. 
To address this bottleneck, recent research has explored the possibility of performing lightweight, pre-training model selection by analyzing the model's intrinsic architecture, rather than relying on full downstream training runs. For instance,~\citet{mirzadeh2022wide} demonstrated that simply enlarging the network width significantly mitigates catastrophic forgetting. They hypothesized that wider networks naturally exhibit greater gradient orthogonality and align with the ``lazy training'' regime~\citep{chizat2019lazy}, where parameters move minimally from their initialization, thus preserving learned representations. Mathematically, this lazy training dynamic in overparameterized networks is closely characterized by the Neural Tangent Kernel (NTK) theory~\citep{jacot2018neural, lee2019wide}. Under this infinite-width limit, several theoretical works have analyzed the generalization guarantees of gradient projection methods~\citep{bennani2020generalisation} and have mathematically bounded catastrophic forgetting through the NTK overlap matrix~\citep{doan2021theoretical}. Extending beyond these idealized infinite-width regimes to practical settings,~\citet{lu2024revisiting} conducted a comprehensive empirical analysis of various finite-width architectural dimensions (such as depth, activation functions, and normalization layers) to map their impact on CL metrics.

These macroscopic studies successfully established the empirical link between architecture and downstream CL metrics (e.g. average accuracy and forgetting), which are too coarse to realize and analyze the trade-off between plasticity (the ability to acquire new knowledge) and stability (the preservation of historical knowledge). To theoretically understand how architecture intrinsically influences the trade-off, recent literature~\citep{benjamin2018measuring,guha2024diminishing,kim2025measuring} has shifted focus from macroscopic performance to the model's logit space, utilizing logit shift as a key metric, which denotes the discrepancy between the logits of past classes predicted by the current model and those predicted by the model immediately after learning those classes.
The work most closely related to ours is~\citet{guha2024diminishing}, which initiated the formal theoretical analysis of these phenomena. They leveraged perturbation theory to analyze Feed-Forward Networks (FFNs), proving that increasing network width directly bounds the logit shift on prior tasks, although with diminishing returns. 

% Prior studies~\citep{guha2024diminishing, kim2025measuring} have analyzed logit shift as a finer-grained object than downstream output metrics. Consequently, balancing the trade-off of ``Plasticity-Stability" in logit space is a critical problem naturally spawned by catastrophic forgetting.

Despite these valuable advancements, a major theoretical gap remains: existing studies~\citet{guha2024diminishing,kim2025measuring} rely on the simplification of a \textbf{homogeneous architecture}, i.e., assuming uniform hidden layer widths across the network. However, real-world models are highly heterogeneous, featuring varying layer-wise widths and complex block depths. 
Our work directly addresses this gap by presenting a theoretical framework designed for heterogeneous architectures. By conceptually decoupling the logit shift into an architecture dependency and a data dependency, we propose the Architecture-driven Shift (ADS), which provides a theoretically grounded, computationally efficient proxy that captures the trajectory of logit shift without full training, facilitating large-scale model selection.

\section{Preliminary}
\subsection{Notation}
Consider a Fully-Connected Neural Network (FNN) $f_{t}$ trained on the $t$-th task, with $L$ layers and hidden layer widths $\boldsymbol{w}=[w^{(1)},\dots,w^{(L)}]$. Let $w^{(0)}=\text{dim}_{in}$ and $w^{(L+1)}=\text{dim}_{out}$ denote the input and output dimensions, respectively. For any layer $l\in \{1, \dots, L\}$, we define forward propagation and backpropagation as follows. During forward propagation, the input $\boldsymbol{x}$ is passed through each layer $l$ to compute the activation $\boldsymbol{a}^{(l)} = \sigma(\boldsymbol{z}^{(l)}) = \sigma(\Theta^{(l)} \boldsymbol{a}^{(l-1)})$, where $\Theta^{(l)}$ is the weight matrix of layer $l$, and $\sigma$ is the activation function. Following the convention in this area~\cite{guha2024diminishing}, we ignore the bias terms for simplicity. In the backpropagation phase, we define the error signal $\delta^{(l)}$ of the $l$-th layer, which represents the gradient of the loss $\mathcal{L}$ with respect to the layer's pre-activation $\boldsymbol{z}^{(l)}$. 
% Based on the chain rule, the error signal can be recursively computed starting from the output layer and propagated backward through the network. (see Proposition~\ref{app_lemma:the_recursive_formula_for_the_error_signal})     \begin{align}
%     \delta^{(l)} = \Theta^{(l+1)\top} \delta^{(l+1)} \odot \sigma'(z^{(l)}),
% \end{align}
% where $\odot$ represents the Hadamard product (element-wise multiplication), $(\cdot)^{\top}$ and $\sigma'(\cdot)$ represent the transpose and the first-order derivative of the activation function, respectively.
Moreover, we define the sensitivity of the true-class logit after training on task $t$ for $S$ steps as $\boldsymbol g_{\mathrm{old}}:=\nabla_\Theta f_t^{(y)}(\boldsymbol{x})$ where $f_t^{(y)}(\boldsymbol{x})$ denotes the scalar logit associated with the ground-truth class $y$. We further define the update direction at the $s$-th training step on task $i+1$ as $\boldsymbol{g}_{new}(s) := \nabla_{\Theta_{(t+1, s)}} \mathcal{L}_{t+1}$. To quantify the microscopic task conflict during adaptation (as mentioned in Section~\ref{sec:introduction}), let $\cos \theta^{(l)}$ be their cosine similarity in the $l$-th layer.

\subsection{Problem Setup}
\label{sec:problem_setup}
Formally, we formulate the process of continual learning (CL) as a sequential learning task. A heterogeneous ReLU FNN, parameterized by $\Theta$, with variable width $\boldsymbol{w}$ and depth $L$, is trained on a stream of datasets $D_{1},\dots, D_{T}$. Each task $t$ involves $S$ steps of supervised learning on $D_{t}=(\mathcal{X}_t, \mathcal{Y}_t)$, and we assume $S$ steps of training enable the model to perform well on $D_{t}$. Under this sequential setting, we aim to learn a ReLU FNN that requires navigating the fundamental ``Plasticity-Stability" trade-off. Specifically, the model must exhibit the ability to adapt to each new task $D_t$ while retaining performance on previous tasks $D_{1:t-1}$, without explicit access to their data.  To mathematically quantify the trade-off, we adopt the definition of \textit{logit shift} originally proposed by~\citet{benjamin2018measuring, guha2024diminishing} as our core analytical objective, which is defined as the expected variation in the logit space:
\begin{equation}\label{eq:logit_shift}
\mathcal{S}_{t \to t+1} = \mathbb{E}_{\boldsymbol{x}\sim D_{1:t}} \left[ \| f_{t+1}(\boldsymbol{x}; \Theta_{t+1})-f_{t}(\boldsymbol{x}; \Theta_t) \|_2 \right],
\end{equation}
where $f_{t}(\boldsymbol{x}; \Theta_t)$ is simplified as $f_{t}(\boldsymbol{x})$ for brevity in the following discussion.

\textbf{Our Objective.}
However, computing $\mathcal{S}_{t \to t+1}$ requires a full training on the complete dataset $D_T$, making it computationally prohibitive for large-scale model selection. Furthermore, we recognize that $\mathcal{S}_{t \to t+1}$ inherently contains an architectural dependency that may offer a theoretical pathway to capture the logit shift without conducting the full training on the complete dataset.

However, since the training process is fundamentally driven by data, a pure decoupling of architecture and data is theoretically unattainable. Instead, our strategy is to explicitly extract the architecture dependency (governed by $\boldsymbol{w}$ and $L$) while capturing the unavoidable data-dependent factors of optimization dynamics through a lightweight calibration process on a subset of data.

Therefore, our primary objective is to formulate a proxy, Architecture-Driven Shift (ADS), that captures the tendency of $\mathcal{S}_{t \to t+1}$ primarily using the architectural variables ($\boldsymbol{w}$ and $L$). Mathematically, we aim to conceptualize a decomposition function $\mathcal{F}$ such that:
\begin{equation}
\mathcal{S}_{t \to t+1} \propto \text{ADS} = \mathcal{F}\Big(\boldsymbol{w}, L, D_{sub}\Big),
\end{equation}
where $D_{sub} \subset D_{1:t}$ is a subset used to calibrate the data-dependent factors. Formalizing $\mathcal{F}$ enables us to evaluate the stability/plasticity of models with heterogeneous architectures in a lightweight manner.

\subsection{Theoretical Groundings}
Before presenting our main theoretical framework, we introduce the necessary theoretical groundings as follows. First, to ensure the network operates in a trainable regime where the aforementioned error signals do not vanish or explode exponentially with depth, we rely on the following standard prior:

\begin{assumption}[Non-degeneracy of Layer-wise Weights]
    \label{ass:non-degeneracy_of_layer-wise_weight}
    Consider an $L$-layer FNN with a Lipschitz continuous activation function $\sigma(\cdot)$ (where $\sigma(0)=0$) and zero biases. The layer-wise weight matrices $\Theta^{(l)}$ follow a non-degenerate distribution with variance satisfying:
    \begin{align}
        \operatorname{Var}(\Theta^{(l)}) \gg 0, \qquad \forall \, l=1,\dots,L \iff \mathbb{E}[(\Theta^{(l)})^2] \gg \big(\mathbb{E}[\Theta^{(l)}] \big)^2.
    \end{align}
    This assumption is theoretically grounded in standard initialization and guarantees second-order isotropy (further details in Assumption~\ref{app_ass:non-degeneracy_of_layer-wise_weight}).
\end{assumption}

Furthermore, microscopically analyzing optimization process in CL scenarios requires understanding how gradients from sequential tasks interact within the high-dimensional space. To model this gradient alignment phenomenon, we introduce next lemmas:

\begin{lemma}[Levy's Lemma on the Sphere~\citep{vershynin2018high}]
    \label{lemma:levys_lemma}
    Let $\boldsymbol{v}_{\text{noise}}$ be a random vector uniformly distributed on the unit sphere in the effective parameter space $\mathbb{R}^{D_{\text{eff}}^{(l)}}$. For any $1$-Lipschitz function $F$, such as the gradient projection $F(\boldsymbol{v}) = \boldsymbol{g}_{\text{old}}^\top \boldsymbol{v}$ (assuming $\|\boldsymbol{g}_{\text{old}}\|_2=1$), and for any scalar $\epsilon > 0$, the probability measure satisfies:
    \begin{align}
        \label{eq:levys_lemma_application}
        P\left( |\boldsymbol{g}_{\text{old}}^\top \boldsymbol{v}_{\text{noise}}| \geq \epsilon \right) \leq 2 \exp\left( -\frac{D_{\text{eff}}^{(l)} \epsilon^2}{2} \right).
    \end{align}
\end{lemma}

\begin{lemma}[Asymptotic Vanishing of the Effective Dimension]
    \label{lemma:asymptotic_vanishing} 
    Let the gradient space at layer $l$ admit the orthogonal decomposition $\mathbb{R}^{N^{(l)}} = V_{\mathrm{signal}} \oplus V_{\mathrm{noise}}$, where $N^{(l)} = w^{(l)} w^{(l-1)}$ denotes the nominal parameter dimension. Suppose the signal subspace dimension $\dim(V_{\mathrm{signal}}) = k$ is independent of model width. Then in the wide-network limit $w^{(l)} \to \infty$,
    \begin{align}
        \lim_{w^{(l)} \to \infty} \sqrt{\frac{\log(L/\delta)}{D_{\text{eff}}^{(l)}}} = 0.
    \end{align}
\end{lemma}
The full proof details is offered in the appendix (Lemma~\ref{app_lemma:asymptotic_vanishing}).

Based on the groundings above, we now have enough theoretical tools to model three critical theoretical components: (1) the scaling of the spectral norm of weight gradients with layer width, (2) the relationship between architecture and the optimization path length, and (3) the asymptotic nature of task conflict in wide networks. Therefore, the expression of ADS can be derived in Section~\ref{sec:theoretical_framework} depends on these components. 

\section{Theoretical Framework Estimation} \label{sec:theoretical_framework}
\subsection{Formulation: Decomposing the logit shift}
Our primary objective is to formulate a proxy, Architecture-Driven Shift (ADS), that captures the tendency of $\mathcal{S}_{t \to t+1}$ primarily using the architectural variables ($\boldsymbol{w}$ and $L$). Because we don't have any prior knowledge about ADS, we try to model it starting from the logit shift, which is clearly defined in~\eqref{eq:logit_shift}. Furthermore, under the local linearization assumption commonly adopted in deep learning theory analyses~\citet{jacot2018neural, lee2019wide}, we analyze the logit shift using a first-order Taylor approximation around the parameters $\Theta_i$:
    \begin{align}
        \| f_{i+1}(\boldsymbol{x}) - f_{i}(\boldsymbol{x}) \|_2 \approx \| \nabla_{\Theta} f_i(\boldsymbol{x})^\top \Delta \Theta_i \|_2.
    \end{align}
Substituting the gradient flow dynamics $\Delta \Theta_i = \int_0^S \nabla_{\Theta} \mathcal{L}_{i+1}(\Theta_s) ds$, and applying the triangle inequality for integrals, we obtain:
    \begin{align}
        \| f_{i+1}(\boldsymbol{x}) - f_{i}(\boldsymbol{x}) \|_2 
        &= \left\| \int_0^S \langle \nabla_{\Theta} f_i(\boldsymbol{x}), \nabla_{\Theta} \mathcal{L}_{i+1}(s) \rangle ds \right\|_2 \\
        &\leq \int_0^S \left| \sum_{l=1}^L \langle \boldsymbol{g}_{\text{old}}^{(l)}, \boldsymbol{g}_{\text{new}}^{(l)}(s) \rangle \right| ds \label{eq:gradient_alignment_dynamics}\\
        &\leq \sum_{l=1}^L \underbrace{\|\boldsymbol{g}_{\text{old}}^{(l)}\|_2 \int_0^S \|\boldsymbol{g}_{\text{new}}^{(l)}(s)\|_2 \cdot |\cos \theta_{s}^{(l)}| ds}_{\text{Term}^{(l)}},
    \end{align}
    where $\theta_{s}^{(l)}$ is the angle between the old task gradient $\boldsymbol{g}_{\text{old}}^{(l)}$ and the time-varying new task gradient $\boldsymbol{g}_{\text{new}}^{(l)}(s)$. Next, we will analyze each term in $\text{Term}^{(l)}$ independently, with each term corresponding to a distinct theoretical component.
    
\subsection{Static Structural Property: Gradient Norm Scaling with Layer Width}
\begin{proposition} (Gradient Scaling for Weight Matrix)
    \label{prop:gradient_scaling_for_weight_matrix}
    The spectral norm of the gradient with respect to the entire weight matrix $\Theta^{(l)}$ scales as:
    \begin{align}
        \|\nabla_{\Theta^{(l)}} f_i(x)\|_2 \propto \sqrt{w^{(l-1)}}.
    \end{align}
    \begin{proof}
        See detailed proof process in Appendix~\ref{app_prop:gradient_scaling_for_weight_matrix}.
    \end{proof}
\end{proposition}

\subsection{Macroscopic Optimization Dynamics: Scaling of the Optimization Path}
\textbf{Empirical Observation (Width Scaling):} 
\begin{figure}[h]
\centering
\includegraphics[width=\linewidth]{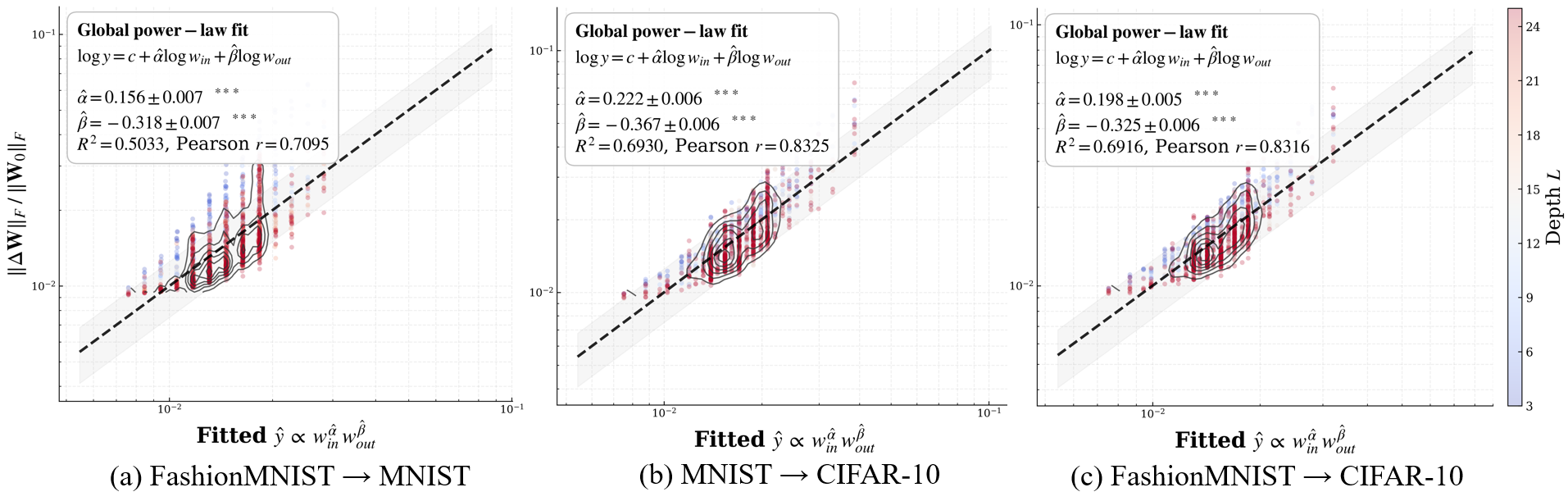}
\caption{Empirical Validation of the Width Term Assumption.}
\label{fig:width_term_assumption}
\end{figure}

Inspired by the empirical observation in Figure~\ref{fig:width_term_assumption}, we extend the width scaling assumption in~\citep{guha2024diminishing} to satisfies the variable width of real-world model as follows. \footnote{We provide more details in Appendix~\ref{app_sec:more_evidence_of_ass1}}.

\begin{assumption}
\label{ass:width_term}
Let $\Theta^{(l)}_{(t,S)}\in \mathbb{R}^{w^{(l)}\times w^{(l-1)}}$ represents the weight matrix of the $l$-th layer that has been updated $S$ steps on the t-th task. 
\begin{align}\label{equ:width_term}
\frac{\|\Theta^{(l)}_{(t,S)} - \Theta^{(l)}_{(t-1,S)}\|_F}{\|\Theta^{(l)}_{(t-1,S)}\|_F}\triangleq  \frac{\|\Delta\Theta^{(l)}_{(t,S)}\|_F}{\|\Theta^{(l)}_{(t-1,S)}\|_F} \propto w^{(l-1)\,\alpha}w^{(l)\,\beta},
\end{align}
where $\|\cdot \|_F$ denotes the Frobenius norm. $\alpha$ and $\beta$ are task-specific exponential factor.
\end{assumption}

\begin{lemma}
\label{lemma:pathlen_scaling_law}
To quantify the distance of the optimization trajectory, we define $\mathrm{PathLen}^{(l)} \triangleq \int_0^S \|\nabla_{\Theta_{(i+1, r_k)}^{(l)}} \mathcal{L}_{i+1}\|_F dr_k$, i.e., the integral of the Frobenius norm over $S$ training steps. Under Assumptions~\ref{ass:width_term} and~\ref{trajectory regularity}, the path length yields the stated proportionality:
\begin{align}
    \mathrm{PathLen}^{(l)} \propto w^{(l-1)\,\alpha}{w^{(l)\, \beta}}.
\end{align}
\begin{proof}
    We define $\mathrm{Disp}^{(l)}$ as $\|\Delta \Theta\|_F = \left\|\int_0^S \nabla_{\Theta_{(i+1, r_k)}^{(l)}} \mathcal{L}_{i+1} dr_k \right\|_F$ represents the displacement during model optimization. Based on Assumptions~\ref{ass:width_term}, $\mathrm{Disp}^{(l)}$ satisfies:
    \begin{align}
        \mathrm{Disp}^{(l)} \propto C_\gamma w^{(l-1)\,\alpha}w^{(l)\,\beta},
    \end{align}
    where $C_\gamma$ is constant related with the Frobenius norm of the weight matrix of the previous task. Under Assumption~\ref{trajectory regularity}, we have:
    \begin{align}
        \mathrm{PathLen}^{(l)} \approx C_{\mathrm{traj}} \cdot\mathrm{Disp}^{(l)} \propto w^{(l-1)\,\alpha}{w^{(l)\, \beta}},
    \end{align}
    where $C_{\mathrm{traj}}$ is constant close to 1 ($\in (1, 1.1)$ in our experiments), associating with the trajectory.
\end{proof}
\end{lemma}

\subsection{Microscopic Optimization Dynamics: Asymptotic Purification of Task Conflicts}

\textbf{Empirical Observation (The Middle-Layer Vulnerability):}
\begin{figure}[h]
\centering
\includegraphics[width=\linewidth]{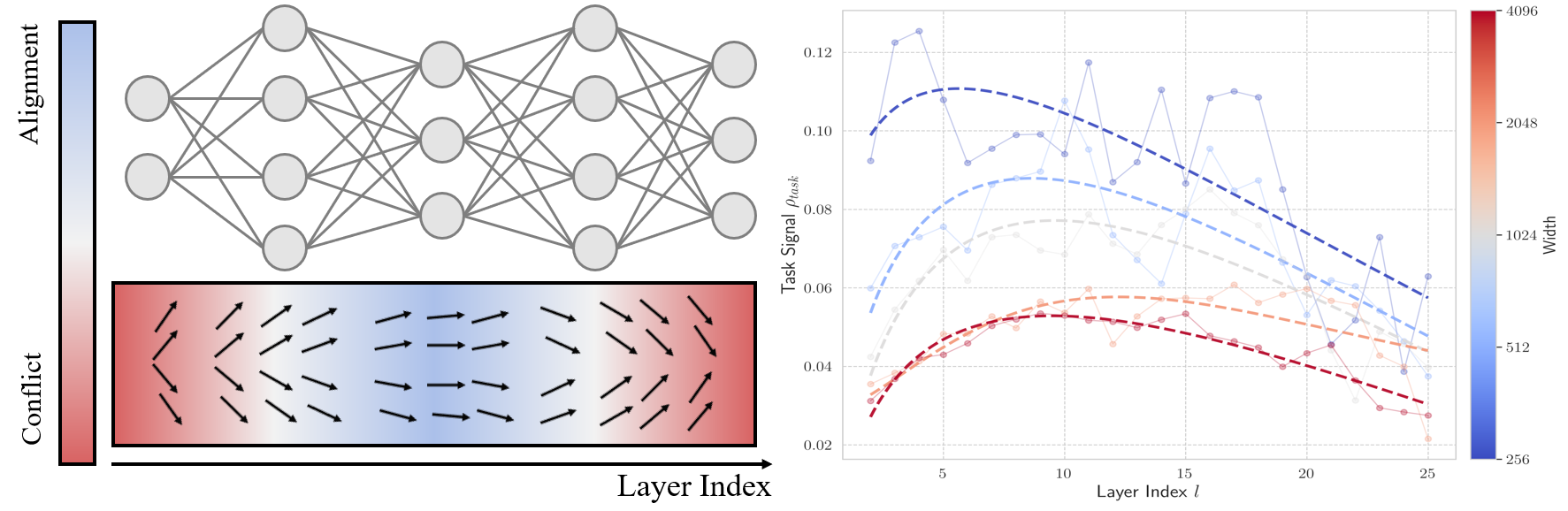}
\caption{Empirical Validation of the Middle-Layer Vulnerability.}
\label{fig:middle_layer_vulnerability}
\end{figure}

Motivated by observations in Figure~\ref{fig:middle_layer_vulnerability}, we introduce a assumption that characterize the ``U-shape curve", which is defined as \textbf{middle-layer vulnerability}. \footnote{There are more empirical evidence in Appendix~\ref{app_sec:more_evidence_of_ass2} to support Assumption~\ref{ass:depth_term}.}

\begin{assumption}
    \label{ass:depth_term}
    Let $\cos \theta$ be the cosine similarity of the output sensitivity $\boldsymbol{g}_{old}$ and update direction $\boldsymbol{g}_{new}(\tau)$.
    \begin{align}
        \rho_{task}^{(l)}\propto l^b \cdot e^{-cl},
    \end{align}
    where $b$ and $c$ is a data-dependent parameters.
\end{assumption}

\begin{proposition} (Signal-Noise Decomposition under High Dimensionality)
    \label{prop:signal_noise}
    Under the Signal-Noise Decomposition hypothesis, the cosine similarity between output sensitivity $\boldsymbol{g}_{old} \triangleq \nabla_{\Theta_{(i, S)}} f_{i}(x)$ and update direction $\boldsymbol{g}_{new}(\tau) \triangleq \nabla_{\Theta_{(i+1, r_k)}} \mathcal{L}_{i+1}$ at layer $l$ can be decomposed into a deterministic task-alignment term and a stochastic high-dimensional noise term. With probability at least $1-\delta$, the cosine similarity satisfies the following scaling law:
    \begin{align}
         \cos \theta^{(l)} \propto \left(|\rho_{\text{task}}^{(l)}| 
            +
            \sqrt{\frac{\log(L/\delta)}{D_{\text{eff}}^{(l)}}} \right),
    \end{align}
    where $D_{\text{eff}}^{(l)}$ denotes the effective dimension of the parameter space at layer $l$, and $\rho_{\text{task}}^{(l)}$ represents the semantic conflict component.
    \begin{proof}
        See in Appendix~\ref{app_prop:signal_noise}.
    \end{proof}
\end{proposition}

Based on Proposition~\ref{prop:signal_noise} above and Lemma~\ref{lemma:asymptotic_vanishing}, the gradient cosine similarity is dominated by the deterministic signal component, which satisfy Assumption~\ref{ass:depth_term}, rendering the random high-dimensional noise negligible. The relationship simplifies to:
\begin{align}
    \label{eq:theorem_depth_term}
    \cos \theta^{(l)} \approx \rho_{\text{task}}^{(l)} \propto l^b \cdot e^{-cl}.
\end{align}

\textbf{Remark:} This implies that in wide layers, any significant gradient interference ($\cos \theta^{(l)} \neq 0$) stems mainly from semantic task conflict ($\rho_{\text{task}}^{(l)}$), not random parameter collisions.

\subsection{Synthesis: Deriving the Architecture-Driven Shift (ADS)}
Incorporating Lemma~\ref{lemma:pathlen_scaling_law} ($\mathrm{PathLen}^{(l)} \propto w^{(l-1)\,\alpha}{w^{(l)\, \beta}}$) and~\eqref{eq:theorem_depth_term},  the total contribution of layer $l$ satisfies:
    \begin{align}
        \text{Term}^{(l)} 
        \lesssim \sqrt{w^{(l-1)}} \cdot w^{(l-1)\,\alpha}\cdot {w^{(l)\, \beta}} \cdot |l^b e^{-cl}|
        = (w^{(l-1)})^{\alpha+\frac{1}{2}} (w^{(l)})^{\beta} \cdot |l^b e^{-cl}|.
    \end{align}

Therefore, the logit shift satisfies:
    \begin{align}
        \| f_{i+1}(\boldsymbol{x}) - f_{i}(\boldsymbol{x}) \|_2 
        \leq \sum_{l=1}^{L} \text{Term}^{(l)} 
        \lesssim \sum_{l=1}^{L}(w^{(l-1)})^{\alpha+\frac{1}{2}} (w^{(l)})^{\beta} \cdot |l^b e^{-cl}|. \label{eq:exact_closed-form_result}
    \end{align}
    where $\alpha, \beta, b, c$ are task-specific parameters that are calibrated from a subset of Dataset $D_{sub} \subset D$.

In conclusion, we can abstract the right-hand side of~\eqref{eq:exact_closed-form_result} into $\mathcal{F}\Big(\boldsymbol{w}, L, D_{sub}\Big)$. This formulation captures the magnitude of $\mathcal{S}_{t\to t+1}$ primarily through the architectural variables $\boldsymbol{w}$ and $L$, yielding the following relationship:
\begin{equation}
\underbrace{\mathbb{E}_{\boldsymbol{x}\sim D_{1:t}} \left[ \| f_{t+1}(\boldsymbol{x}; \Theta_{t+1})-f_{t}(\boldsymbol{x}; \Theta_t) \|_2 \right]}_{\mathcal{S}_{t \to t+1}}
\propto 
\underbrace{\sum_{l=1}^{L}(w^{(l-1)})^{\alpha+\frac{1}{2}} (w^{(l)})^{\beta} \cdot |l^b e^{-cl}|}_{\text{ADS}},
\label{eq:abstract_ADS_to_answer_question}
\end{equation}

This expression fulfills the objective set forth in the Problem Setup 
(Section~\ref{sec:problem_setup}), which enables direct navigation of the ``Plasticity-Stability" trade-off across FNNs with heterogeneous architectures, without requiring full learning on the new task.

\textbf{Remark:} ADS term of~\eqref{eq:abstract_ADS_to_answer_question} can be conceptually decomposed into data dependency and architecture dependency. Specifically, $\alpha, \beta, b, c$ are data dependency because they are calibrated from a subset of Dataset $D_{sub}$; $w, l, L$ are architecture dependency.

\section{Experiments}
We organize comprehensive experiments to verify our theorem that the proposed architecture-driven shift (ADS) captures the trend of the logit shift. What's more, we analysis the correlation between ADS and Expected Calibration Error (ECE), proving that ADS is useful to decrease computational cost in reliable continual learning model selection. (See Appendix~\ref{app_sec:application} for more detailed results.)

\paragraph{Evaluation Metrics.}
To evaluate how well ADS captures the trend of the logit shift, we employ classic statistical metrics (Spearman's coefficient $r_s$, Kendall's coefficient $r_k$) along with the 95\% confidence interval (CI), and introduce a new metric, \emph{direction consistency} (DC):
\begin{align}
    DC = \frac{2 \sum_{1 \le i < j \le N_{arc}}}{N_{arc}(N_{arc}-1)} \mathrm{1}\left[\Delta ADS_{ij} \cdot \Delta S_{ij} > 0 \right],
\end{align}
where $N_{arc}$ is the number of architectures, $\mathrm{1}[\cdot]$ is the indicator function, and $\Delta(\cdot)_{ij}$ denotes the difference between the values of the corresponding variable for the $i$-th and $j$-th architectures. A DC of 100\% indicates perfect directional agreement between ADS and the logit shift, whereas 50\% corresponds to the expected agreement under random guessing.

In reliable continual learning, ECE is a traditional metric to evaluate the reliability of models' output in neural network calibration literature~\citet{guo2017calibration, li2024calibration}.
\begin{align}
    \text{ECE} = \sum_{b=1}^{15} \frac{|B_b|}{N} \cdot |\text{conf}(B_b) - \text{acc}(B_b)|,
\end{align}
where $B_b$ denotes the set of sample indices whose predicted confidence falls into the $b$-th bin and $N_{arc}$ means the number of data samples. Here, $\text{acc}(B_b) = \frac{1}{|B_b|}\sum_{i \in B_b} \mathrm{1}(\hat{y}_i = y_i)$ represents the average accuracy of the samples in bin $B_b$, and $\text{conf}(B_b) = \frac{1}{|B_b|}\sum_{i \in B_b} \hat{p}_i$ represents the average predicted confidence within that bin.

\subsection{Theoretical Framework Verification}
\begin{table}[t]
  \vspace{\baselineskip} % one line space before the table title
  \caption{Statistical correlation between \textbf{ADS} and \textbf{logit shift} across two paradigms: transfer learning in 2 tasks and sequence learning in 5 tasks. All results are based on 175 architectures of fully-connected neural network per scenario, averaged over 3 random seeds. For simplicity, we refer to MNIST, Fashion-MNIST, and CIFAR-10 as M, F, and C. *** indicates $p < 0.001$.}
  \label{tab:theorem1_verification}
  \vspace{\baselineskip} % one line space after the table title
  \begin{center}
    \begin{small}
      \begin{sc}
        \begin{tabular}{c c c c c}
          \toprule
          \multirow{2}{*}{\shortstack{\textbf{Scenario}}}
          & \multirow{2}{*}{\shortstack{\textbf{Spearman $r_s$}}}
          & \multirow{2}{*}{\shortstack{\textbf{Kendall $r_k$}}}
          & \multirow{2}{*}{\shortstack{\textbf{Direction}\\\textbf{Consistency}}}
          & \multirow{2}{*}{\shortstack{\textbf{95\% CI}}} \\
          \multicolumn{1}{c}{} & & & & \\
          \midrule
          M $\rightarrow$ F
          & 0.920*** & 0.761*** & 88.06\% & [89.1\%, 91.7\%] \\

          M $\rightarrow$ C
          & 0.881*** & 0.683*** & 84.14\% & [81.7\%, 86.2\%] \\
          
          F $\rightarrow$ M
          & 0.955*** & 0.810*** & 90.52\% & [85.6\%, 90.2\%] \\

          F $\rightarrow$ C
          & 0.731*** & 0.526*** & 76.32\% & [73.2\%, 79.5\%] \\

          \midrule
          Split M
          & 0.738*** & 0.560*** & 78.00\% & [73.5\%, 81.7\%] \\

          Split F
          & 0.814*** & 0.631*** & 82.20\% & [78.8\%, 85.1\%] \\
          
          Rotated M
          & 0.802*** & 0.619*** & 80.96\% & [77.1\%, 84.2\%] \\

          Rotated F
          & 0.960*** & 0.833*** & 91.67\% & [89.8\%, 93.1\%] \\

          \midrule

          Across scenarios & 0.850 & 0.678 & 83.98\% & - \\

          \bottomrule
        \end{tabular}
      \end{sc}
    \end{small}
  \end{center}
  \vspace{\baselineskip} % one line space after the table
\end{table}

As demonstrated in Table~\ref{tab:theorem1_verification}, across both transfer learning and sequence CL learning settings, ADS exhibits consistently strong and significant monotonic correlation with logit shift, together with high DC, indicating that the ADS captures the trend of logit shift across architectures. Furthermore, comparing to the~\citeauthor{guha2024diminishing}, ADS perform a stronger correlation with observed logit shift in Vision Transformer (ViT) as shown in Figure~\ref{app_fig:vit_theory_logit_shift} of Appendix, suggesting that the effectiveness of ADS in adopting heterogeneous architecture. 

\begin{table}[t]
  \vspace{\baselineskip}
  \caption{Spearman correlations $r_s$ between the fitted parameters and the logit shift across different scenarios. For simplicity, we refer to MNIST, FashionMNIST, and CIFAR-10 as M, F, and C, respectively. The parameters fitted on representative source tasks under the same shift regime (e.g., M $\rightarrow$ F for \emph{small shift} and M $\rightarrow$ C for \emph{large shift}) already transfer effectively to unseen scenarios, yielding strong correlations. When a data subset is available (30\%--80\%), calibrating the parameters on that subset further improves the correlation, suggesting that the parameters are highly transferable and can be efficiently estimated in practice.}
  \label{tab:parameter_calibration}
  \vspace{\baselineskip}
  \begin{center}
    \begin{small}
      \begin{sc}
        \begin{tabular}{lccccccccc}
        \toprule
        \multirow{2}{*}{\shortstack{\textbf{Scenarios}}}
        & \multirow{2}{*}{\shortstack{\textbf{Small}\\\textbf{Shift}}}
        & \multirow{2}{*}{\shortstack{\textbf{Large}\\\textbf{Shift}}}
        & \multirow{2}{*}{\shortstack{\textbf{30\%}}}
        & \multirow{2}{*}{\shortstack{\textbf{40\%}}}
        & \multirow{2}{*}{\shortstack{\textbf{50\%}}}
        & \multirow{2}{*}{\shortstack{\textbf{60\%}}}
        & \multirow{2}{*}{\shortstack{\textbf{70\%}}}
        & \multirow{2}{*}{\shortstack{\textbf{80\%}}}
        & \multirow{2}{*}{\shortstack{\textbf{Average}\\\textbf{Std}}} \\
        \multicolumn{1}{c}{} & & & & & & & & & \\
        \midrule
          \shortstack[l]{F (Rot $0^\circ$\\$\to 22.5^\circ$)}
          & 0.946 & 0.890 & 0.942 & 0.931 & 0.935 & 0.920 & 0.930 & \textbf{0.956}
          & \shortstack{$0.931$\\$_{\pm 0.020}$} \\[6pt]
          \shortstack[l]{M (Rot $0^\circ$\\$\to 22.5^\circ$)}
          & 0.660 & 0.600 & 0.638 & 0.594 & 0.622 & 0.600 & \textbf{0.672} & 0.618
          & \shortstack{$0.626$\\$_{\pm 0.029}$} \\[6pt]
          M $\rightarrow$ F
          & 0.920 & 0.832 & \textbf{0.932} & 0.920 & 0.915 & 0.919 & 0.913 & 0.879
          & \shortstack{$0.904$\\$_{\pm 0.033}$} \\[4pt]
          F $\rightarrow$ M
          & 0.956 & 0.850 & \textbf{0.959} & 0.943 & 0.956 & 0.954 & 0.952 & 0.927
          & \shortstack{$0.937$\\$_{\pm 0.037}$} \\[4pt]
          M $\rightarrow$ C
          & 0.785 & 0.881 & 0.884 & 0.877 & 0.888 & 0.866 & 0.854 & \textbf{0.891}
          & \shortstack{$0.866$\\$_{\pm 0.035}$} \\[4pt]
          F $\rightarrow$ C
          & 0.382 & 0.535 & 0.691 & 0.685 & 0.654 & \textbf{0.828} & 0.663 & 0.720
          & \shortstack{$0.645$\\$_{\pm 0.133}$} \\[4pt]
          Split M
          & 0.630 & 0.525 & 0.755 & 0.705 & 0.649 & 0.639 & 0.590 & \textbf{0.789}
          & \shortstack{$0.660$\\$_{\pm 0.086}$} \\[4pt]
          Split F
          & 0.685 & 0.480 & 0.844 & 0.860 & 0.841 & 0.045 & 0.822 & \textbf{0.895}
          & \shortstack{$0.684$\\$_{\pm 0.291}$} \\[4pt]
          Rotated M
          & 0.530 & 0.445 & 0.751 & 0.780 & 0.674 & 0.758 & 0.806 & \textbf{0.861}
          & \shortstack{$0.701$\\$_{\pm 0.143}$} \\[4pt]
          Rotated F
          & \textbf{0.976} & 0.879 & 0.951 & 0.954 & 0.969 & 0.966 & 0.956 & 0.949
          & \shortstack{$0.950$\\$_{\pm 0.030}$} \\[4pt]
        \midrule
          \shortstack[l]{Across\\Scenarios}
                 & \shortstack{$0.747$\\$_{\pm 0.203}$}
                 & \shortstack{$0.692$\\$_{\pm 0.189}$}
                 & \shortstack{$0.835$\\$_{\pm 0.118}$}
                 & \shortstack{$0.825$\\$_{\pm 0.127}$}
                 & \shortstack{$0.810$\\$_{\pm 0.143}$}
                 & \shortstack{$0.750$\\$_{\pm 0.278}$}
                 & \shortstack{$0.816$\\$_{\pm 0.132}$}
                 & \shortstack{$0.849$\\$_{\pm 0.109}$}
                 & - \\
        \bottomrule
        \end{tabular}
      \end{sc}
    \end{small}
  \end{center}
  \vspace{\baselineskip}
\end{table}

As shown in Table~\ref{tab:parameter_calibration}, the fitted parameters are highly transferable across scenarios: parameters learned from a representative source task already provide strong zero-shot capture of logit shift tendency, and calibration on only a small proportion of data further commonly improves the correlation, demonstrating that ADS can be efficiently adapted in practice with low computational requirements.

\subsection{Model Selection in Reliable Continual Learning}

\begin{figure}[h]
\centering
\includegraphics[width=\linewidth]{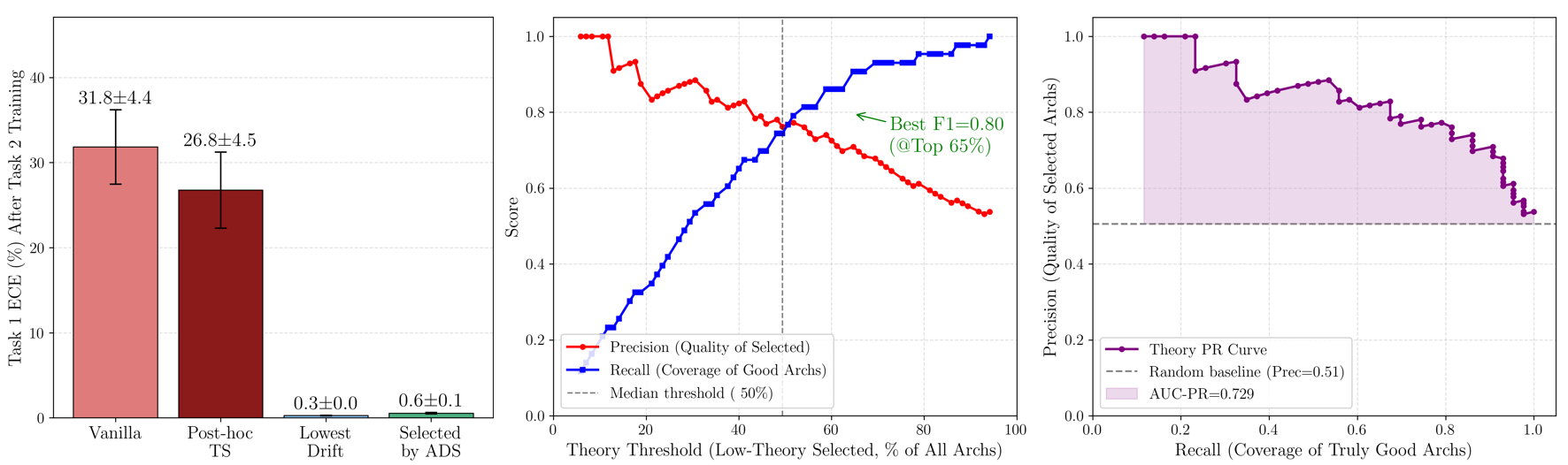}
\caption{Quantitative evaluation of the ADS-based selector: Calibration efficacy, filtering trade-offs, and statistical robustness on Split MNIST (classes 0-4 $\to$ classes 5-9). More validation scenarios are provided in Figure~\ref{fig:application_ece_full} of Appendix.}
\label{fig:application_ece}
\end{figure}

In reliable CL scenario, ECE is a national and representative metric to evaluate the reliability of models' output. Here, we define the ECE drift that less than the median of ECE drift as the positive class, enabling us to analysis the performance of ADS-based selector as shown in Figure~\ref{fig:application_ece}. More details see in Appendix~\ref{app_sec:application}.

Figure~\ref{fig:application_ece} (Left) highlights the importance of the intrinsic stability originated from model architectures. To illustrate the limitations of relying solely on post-hoc adjustments, we apply Post-hoc Temperature Scaling (TS)~\citep{guo2017calibration} to an poorly-calibrated architecture (Vanilla), yielding a marginal reduction in ECE ($31.8\% \to 26.8\%$). In contrast, the model selected by ADS achieves a remarkably low ECE of $0.6\%$, which is statistically comparable to the theoretical optimum in the candidate pool ($0.3\%$). These results demonstrate that resolving calibration drift in CL requires proper architectural selection rather than relying solely on post-hoc algorithmic patches, and that ADS successfully achieves this selection in an efficient manner.

As Figure~\ref{fig:application_ece} (Middle) is shown, ADS enables reliable pruning of the model search space within the low risk of false rejections. Specifically, when selecting only a small proportion of top-ranked architectures, the selector achieves high Precision, which means that the selected candidates are more likely to exhibit low ECE drift. In comparison, a larger threshold bring high Recall. This ensures that almost all stable architectures are successfully preserved within the selected subset, meaning that the discarded models have high ECE drift with high probability.

Figure~\ref{fig:application_ece} (Right) (and Figure~\ref{fig:application_ece_full} shown in Appendix) evaluates the ranking capability and robustness of ADS through a Precision-Recall (PR) analysis. Under diverse settings, our selector consistently yields an Area Under the PR Curve (AUC-PR) ranging from $0.61$ to $0.73$, systematically and significantly outperforming the random baseline of $0.51$. This robust and threshold-independent performance validates that the mathematical formulations of ADS capture the intrinsic property of logit shift, rather than relying on superficial, dataset-specific statistical correlations.

\section{Conclusion}
In this paper, we propose a lightweight proxy Architecture-driven Shift (ADS), which can be computed from architectural variables and a few calibration data samples, capture the trend of logit shift in continual learning (CL) well. Our theoretical framework, which is based on both static structural properties and dynamic optimization behaviors, provide a mechanistic explanation of how architecture shapes logit-space dynamics. Extensive experiments across fully-connected neural network and Transformer demonstrate that ADS is strongly and monotonically correlated with logit shift. Furthermore, empirical validation of the relationship between ADS and Expected Calibration Error (ECE) suggest that ADS can serve as a coarse-grained selector in reliable CL model selection.

\bibliography{iclr2026_conference}
\bibliographystyle{iclr2026_conference}

\appendix
\section{Detailed Theoretical Framework}
\begin{assumption}[Trajectory Regularity]
\label{trajectory regularity}
    The optimization trajectory in parameter space is free from sharp directional reversals, implying that the ratio of path length to displacement remains bounded.
    \begin{align}
        \mathrm{PathLen}^{(l)} \approx C_{\mathrm{traj}} \cdot \mathrm{Disp}^{(l)},
    \end{align}
    where $C_{\mathrm{traj}}$ is constant. There are a growing body of theoretical and empirical work supports this assumption.
    \textbf{Theoretical Motivation.} From optimization theory, \citet{gupta2021path} show that under standard regularity conditions, such as linear convergence or the Polyak–Kurdyka–Łojasiewicz (PKL) condition, the optimization trajectory admits a bounded length relative to the distance to the solution set $\mathbf{X}^*$, i.e., 
    $$
    \mathrm{PathLen}^{(l)} \;\le\; C_{\mathrm{traj}} \cdot \mathrm{Disp}^{(l)},
    $$
    where $C_{\mathrm{traj}}$ depends on problem-specific parameters. Notably, theoretical bounds for PKL functions suggest $C_{\mathrm{traj}} \propto \sqrt{\kappa}$ (where $\kappa$ is the condition number). These results establish that long, highly winding trajectories are theoretically unnecessary under commonly assumed training regimes. 
    
    \textbf{Empirical Evidence.} Empirically, the work~\citet{goodfellow2014qualitatively} demonstrates that, across a wide range of neural architectures, the loss evaluated along the straight line connecting initialization $\mathbf{x}_0$ and the final solution is typically smooth, monotonic, and nearly convex. Their linear interpolation experiments indicate that the 1-D linear subspace captures the majority of parameter variation, suggesting limited deviation from a straight path during optimization.
    More recently, \citet{singh2024hallmarks} analyze actual optimization trajectories and introduce Mean Directional Similarity (MDS) to quantify directional alignment. They report consistently high MDS values (e.g., $\omega \approx 0.76$ for ResNet-50 on ImageNet), indicating strong directional redundancy in parameter updates. Furthermore, their trajectory visualizations show that increasing model scale leads to progressively darker (i.e., higher cosine similarity) trajectory maps, reflecting increasingly straight and directionally aligned optimization paths.
    
    \textbf{Our Observation.} In our experiments, we find this regularity holds consistently, with the ratio stabilizing around $C_{\mathrm{traj}} \in (1, 1.11)$ for our models and tasks, further validating the assumption.
\end{assumption}

\begin{assumption}[Non-degeneracy of Layer-wise Weight]
    \label{app_ass:non-degeneracy_of_layer-wise_weight}
    Consider an $L$-layer fully-connected neural network with activation function $\sigma(\cdot)$ (Lipschitz continuous, $\sigma(0)=0$) and zero biases. The network is initialized in a trainable regime where gradients do not vanish exponentially with depth. Specifically, the layer-wise weights $\Theta^{(l)}$ follow a non-degenerate distribution with variance satisfying:
    \begin{align}
    \operatorname{Var}_{\Theta^{(l)}} \gg 0, \qquad \forall \, l=1,\dots,L \iff \mathbb{E}[\Theta^{(l)2}] \gg \big(\mathbb{E}[\Theta^{(l)}] \big)^2.
    \end{align}
    This assumption is both theoretically grounded and standard in practice, ensuring the network operates in a trainable regime where gradient norms do not vanish exponentially during backpropagation.
    
    \textbf{Theoretical Basis:} Standard initialization schemes~\citet{he2015delving,glorot2010understanding}, are designed specifically to satisfy this variance condition.
    \textbf{Lazy training regime:} Since the He initialization satisfies the aforementioned properties, they are naturally preserved throughout training under the lazy training regime.~\citep{chizat2019lazy}
    \textbf{Empirical Support:} The necessity of this non-degeneracy condition is validated by extensive large-scale experiments. As shown in Figure.~2 within \cite{xiao2018dynamical} and Figure.~1 and Figure.~6 within~\cite{han2025weight}, training dynamics collapse immediately when the variance drops below the critical threshold, whereas satisfying this assumption enables stable training across hundreds of layers.
\end{assumption}

\begin{lemma}[The Recursive Formula for the Error Signal]
    \label{app_lemma:the_recursive_formula_for_the_error_signal}
    We define $\delta^{(l)}=\frac{\partial \mathcal{L}}{\partial z^{(l)}}$, which is the backpropagated error signal of the $l$-th layer. The recursive formula for $\delta^{(l)}$ satisfiies:
    \begin{align}
        \delta^{(l)} = \Theta^{(l+1)\top} \delta^{(l+1)} \odot \sigma'(z^{(l)}),
    \end{align}
    where $\odot$ represents the Hadamard product (element-wise multiplication).
    \begin{proof}
        We utilize the relationship between the total differential of the loss function $\mathcal{L}$ and its gradient. The differential $d\mathcal{L}$ with respect to the pre-activation vectors $z^{(l)}$ and $z^{(l+1)}$ is given by:
        \begin{align}
            d\mathcal{L} &= \left(\frac{\partial \mathcal{L}}{\partial z^{(l)}}\right)^{\top} dz^{(l)} = \delta^{(l)\top} dz^{(l)}, \label{eq:diff_l} \\
            d\mathcal{L} &= \left(\frac{\partial \mathcal{L}}{\partial z^{(l+1)}}\right)^{\top} dz^{(l+1)} = \delta^{(l+1)\top} dz^{(l+1)}. \label{eq:diff_l_plus_1}
        \end{align}
        Recall the forward propagation formula $z^{(l+1)} = \Theta^{(l+1)} \sigma(z^{(l)})$. To find the partial derivative with respect to $z^{(l)}$, we treat the weights $\Theta^{(l+1)}$ as constant (i.e., $d\Theta^{(l+1)} = 0$). Taking the differential of $z^{(l+1)}$ yields:
        \begin{align}
            dz^{(l+1)} &= d\left(\Theta^{(l+1)} \sigma(z^{(l)})\right) \\
            &= \underbrace{(d\Theta^{(l+1)})}_{0} \sigma(z^{(l)}) + \Theta^{(l+1)} d\left(\sigma(z^{(l)})\right) \\
            &= \Theta^{(l+1)} \left(\sigma'(z^{(l)}) \odot dz^{(l)}\right).
        \end{align}
        Substituting this expression into Eq.~\eqref{eq:diff_l_plus_1}:
        \begin{align}
            d\mathcal{L} &= \delta^{(l+1)\top} \left[ \Theta^{(l+1)} \left(\sigma'(z^{(l)}) \odot dz^{(l)}\right) \right]
            = \left( \Theta^{(l+1)\top} \delta^{(l+1)} \right)^{\top} \left(\sigma'(z^{(l)}) \odot dz^{(l)}\right).
        \end{align}
        We apply the vector identity involving the Hadamard product, $\mathbf{u}^{\top} (\mathbf{v} \odot \mathbf{w}) = (\mathbf{u} \odot \mathbf{v})^{\top} \mathbf{w}$. Letting $\mathbf{u} = \Theta^{(l+1)\top} \delta^{(l+1)}$, $\mathbf{v} = \sigma'(z^{(l)})$, and $\mathbf{w} = dz^{(l)}$, we obtain:
        \begin{align}
            d\mathcal{L} = \left( \Theta^{(l+1)\top} \delta^{(l+1)} \odot \sigma'(z^{(l)}) \right)^{\top} dz^{(l)}.
        \end{align}
        By comparing this result with Eq.~\eqref{eq:diff_l}, we identify the gradient $\delta^{(l)}$ as:
        \begin{align*}
            \delta^{(l)} = \Theta^{(l+1)\top} \delta^{(l+1)} \odot \sigma'(z^{(l)}).
        \end{align*}
    \end{proof}
\end{lemma}

\begin{proposition}[Stability of the $\ell_2$ Norm of Error Signal]
    \label{app_prop:stability_of_error_signal}
    For any layer $1 \leq l \leq L$, let $\delta^{(l-1)}$ be the backpropagated error signal of the $l$-th layer. The second moment of $\delta^{(l-1)}$ satisfies the property that the error norm remains constant with respect to the layer width $w^{(l)}$:
    \begin{align} 
        \mathbb{E}\left[\|\delta^{(l-1)}\|_2^2\right] \sim \mathbb{E}\left[\|\delta^{(l)}\|_2^2\right].
    \end{align}
    \begin{proof}
        Recall the recursive formula for the error signal during backpropagation (Lemma~\ref{app_lemma:the_recursive_formula_for_the_error_signal}): $\delta^{(l-1)} = \sigma'(z^{(l-1)}) \odot \Theta^{(l)\top} \delta^{(l)}$, The element-wise representation for the $j$-th component is given by:
        \begin{align}
            \delta_{j}^{(l-1)} = \sigma'(z_j^{(l-1)}) \sum_{i=1}^{w^{(l)}} \Theta_{ij}^{(l)} \delta_{i}^{(l)}.
        \end{align}
        Consequently, the squared $\ell_2$ norm of $\delta^{(l-1)}$ can expanded as:
        \begin{align}
            \|\delta^{(l-1)}\|_2^2 = \sum_{j=1}^{w^{(l-1)}} (\delta_j^{(l-1)})^2 = \sum_{j=1}^{w^{(l-1)}} \left( \sigma'(z_j^{(l-1)}) \sum_{i=1}^{w^{(l)}} \Theta_{ij}^{(l)} \delta_{i}^{(l)} \right)^2.
        \end{align}
        We define the conditional expectation operator $\mathbb{E}_{\Theta}[\cdot \mid \delta^{(l)}, z^{(l-1)}]$, treating $\delta^{(l)}$ and $z^{(l-1)}$ as fixed quantities and taking the expectation solely over the weights $\Theta^{(l)}$. Leveraging the linearity of expectation, we obtain:
        \begin{align}
            \label{the conditional expectation of the error norm}
            \mathbb{E}_\Theta\left[\|\delta^{(l-1)}\|_2^2 \mid \delta^{(l)}, z^{(l-1)}\right]
            &= \sum_{j=1}^{w^{(l-1)}} (\sigma'(z^{(l-1)}_j))^2 \cdot
            \mathbb{E}_\Theta\left[ \left(\sum_{i=1}^{w^{(l)}} \Theta^{(l)}_{i j}\,\delta^{(l)}_i\right)^2 \,\middle|\, \delta^{(l)} \right].
        \end{align}
        The quadratic term inside the expectation can be expanded as:
        \begin{align}
            \left(\sum_{i=1}^{w^{(l)}} \Theta^{(l)}_{i j}\delta^{(l)}_i \right)^2 = \sum_{i=1}^{w^{(l)}} \sum_{p=1}^{w^{(l)}} \Theta^{(l)}_{i j}\delta^{(l)}_i \Theta^{(l)}_{p j}\delta^{(l)}_p.
        \end{align}
        Thus, the expectation over the weights becomes:
        \begin{align}
            \mathbb{E}_\Theta\left[\left(\sum_{i=1}^{w^{(l)}} \Theta^{(l)}_{i j}\,\delta^{(l)}_i\right)^2 \,\middle|\, \delta^{(l)}\right]
            &= \sum_{i=1}^{w^{(l)}} \sum_{p=1}^{w^{(l)}} \delta^{(l)}_i \delta^{(l)}_p \mathbb{E}_{\Theta}\left[\Theta^{(l)}_{i j} \Theta^{(l)}_{p j}\right].
        \end{align}
        Let $\Theta_{\cdot j}^{(l)} = [\Theta_{1 j}^{(l)}, \dots, \Theta_{w^{(l)} j}^{(l)}]^{\top}$ denote the column vector of weights connecting all neurons in layer $k$ to the $j$-th neuron in layer $k-1$. The expression above can be rewritten in vector form as $\delta^{(k)\top} \mathbb{E}_{\Theta}[\Theta^{(l)}_{\cdot j} \Theta^{(k)\top}_{\cdot j}] \delta^{(l)}$.

        We assume the weights satisfy \textbf{second-order isotropy}. Specifically, $\mathbb{E}_{\Theta}[\Theta^{(l)}_{\cdot j} \Theta^{(k)\top}_{\cdot j}] = v_{\Theta}^{(l)} I_{w^{(l)}}$, where $v_{\Theta}^{(l)} = \mathbb{E}[(\Theta_{i j}^{(l)})^2]$ represents the second moment of the weights. This assumption is weaker than requiring weights to be independent and identically distributed (i.i.d.), as it imposes constraints only on the covariance structure (uncorrelatedness) and the second moment, without specifying the underlying distribution (e.g., Gaussian vs. heavy-tailed). Under this assumption, the cross-terms vanish ($\operatorname{Cov}(\Theta_{ij}, \Theta_{pj}) = 0$ for $i \neq p$), yielding:
        \begin{align}
            \mathbb{E}_\Theta\left[\left(\sum_{i=1}^{w^{(l)}} \Theta^{(l)}_{i j}\,\delta^{(l)}_i\right)^2 \,\middle|\, \delta^{(l)}\right]
            &= v_{\Theta}^{(l)} \delta^{(k)\top} I_{w^{(l)}} \delta^{(l)} 
            = v_{\Theta}^{(l)} \| \delta^{(l)} \|_2^2.
        \end{align}
        Substituting this back into Eq.~\eqref{the conditional expectation of the error norm}:
        \begin{align}
            \mathbb{E}_\Theta\left[\|\delta^{(l-1)}\|_2^2 \mid \delta^{(l)}, z^{(l-1)}\right]
            &= v_{\Theta}^{(l)} \| \delta^{(l)} \|_2^2 \cdot \sum_{j=1}^{w^{(l-1)}} (\sigma'(z^{(l-1)}_j))^2.
        \end{align}
        Assuming the pre-activations $z_j^{(l-1)}$ within the same layer are identically distributed, we apply the Law of Large Numbers. As the width $w^{(l-1)}$ becomes sufficiently large:
        \begin{align}
            \frac{1}{w^{(l-1)}}\sum_{j=1}^{w^{(l-1)}} (\sigma'(z_j^{(l-1)}))^2 \longrightarrow \mathbb{E}_z\left[(\sigma'(z^{(l-1)}))^2\right].
        \end{align}
        For ReLU activation functions, where $\sigma'(z) \in \{0, 1\}$, it holds that $(\sigma'(z))^2 = \sigma'(z)$. Letting $c_{\sigma} = \mathbb{E}[\sigma'(z^{(l-1)})]$, the summation converges to $w^{(l-1)} c_{\sigma}$.

        By the Law of Total Expectation:
        \begin{align}
            \mathbb{E}\left[\|\delta^{(l-1)}\|_2^2\right]
            &= \mathbb{E}\left[ v_{\Theta}^{(l)} \| \delta^{(l)} \|_2^2 \cdot w^{(l-1)} c_{\sigma} \right] \\
            &= v_{\Theta}^{(l)} \cdot w^{(l-1)} c_{\sigma} \cdot \mathbb{E}\left[\|\delta^{(l)} \|_2^2\right].
        \end{align}
        Invoking Assumption~\ref{app_ass:non-degeneracy_of_layer-wise_weight}, we have $\mathbb{E}[(\Theta_{ij}^{(l)})^2] \gg (\mathbb{E}[\Theta_{ij}^{(l)}])^2$, implying $v_{\Theta}^{(l)}=\operatorname{Var}(\Theta_{ij})+\left(\mathbb{E}[\Theta_{ij}^{(l)}]\right)^2\approx \operatorname{Var}(\Theta_{ij}^{(l)})$. Because of the He initialization, $\operatorname{Var}(\Theta_{ij}^{(l)}) = \frac{2}{w^{(l)}}$ is satisfied. Furthermore, assuming $z$ is symmetric about 0, for ReLU we have $c_{\sigma} = 1/2$. Substituting these values:
        \begin{align} 
            \mathbb{E}\Big[\|\delta^{(l-1)}\|_2^2\Big]
            &=
            v_{\Theta}^{(l)} \cdot w^{(l-1)}\cdot c_{\sigma} \cdot\mathbb{E}\Big[\|\delta^{(l)} \|_2^2\Big]
            \\
            &= \frac{2}{w^{(l-1)}} \cdot w^{(l-1)} \cdot \frac{1}{2} \cdot \mathbb{E}\Big[\|\delta^{(l)} \|_2^2\Big]
            \\
            &= \mathbb{E}\Big[\|\delta^{(l)} \|_2^2\Big].
        \end{align}
        Accounting for potential noise and finite-width effects, we conclude that the error norms are of the same order:
        \begin{align*}
            \mathbb{E}\left[\|\delta^{(l-1)}\|_2^2\right] \sim \mathbb{E}\left[\|\delta^{(l)}\|_2^2\right].
        \end{align*}
    \end{proof}
\end{proposition}

\begin{lemma}[Activation Scaling Law]
    \label{app_lemma:activation_scaling_law}
    The $\ell_2$ norm of the activation vector at the $l$-th layer scakes proportionally with the square root of the layer width:
    \begin{align}
        \|a^{(l-1)}\|_2 \propto \sqrt{w^{(l-1)}}.
    \end{align}
    \begin{proof}
        Consider the pre-activation $z_i^{(l)}$ for an arbitrary neuron $i \in \{1, \dots, w^{(l)}\}$ in the current layer $l$. It is defined as the linear transformation of the activations from the previous layer:
        \begin{align}
            z_i^{(l)} = \sum_{j=1}^{w^{(l-1)}} \Theta_{ij}^{(l)} a_j^{(l-1)}.
        \end{align}
        We analyze the variance of $z_i^{(l)}$ conditioned on the input activations $a^{(l-1)}$. Utilizing the linearity of variance for independent variables (invoking the weak isotropy assumption where weights are uncorrelated, i.e., $\operatorname{Cov}(\Theta_{ij}, \Theta_{ik}) \approx 0$ for $j \neq k$), the cross-terms vanish. 

        Furthermore, applying assumption~\ref{app_ass:non-degeneracy_of_layer-wise_weight}, we have $\mathbb{E}[\Theta^2] \gg (\mathbb{E}[\Theta])^2$, i.e., $\operatorname{Var}(\Theta_{ij}^{(l)}) \approx \mathbb{E}[(\Theta_{ij}^{(l)})^2] \triangleq v_\Theta^{(l)}$. Thus, the conditional variance simplifies to:
        \begin{align}
            \operatorname{Var}\left(z_i^{(l)} \mid a^{(l-1)}\right) 
            &= \sum_{j=1}^{w^{(l-1)}} \operatorname{Var}\left( \Theta_{ij}^{(l)} a_j^{(l-1)} \mid a^{(l-1)} \right) \\
            &= \sum_{j=1}^{w^{(l-1)}} (a_j^{(l-1)})^2 \cdot \operatorname{Var}(\Theta_{ij}^{(l)}) \\
            &\approx v_\Theta^{(l)} \sum_{j=1}^{w^{(l-1)}} (a_j^{(l-1)})^2 \\
            &= v_\Theta^{(l)} \|a^{(l-1)}\|_2^2.
        \end{align}
        To ensure stable signal propagation during training (preventing vanishing or exploding gradients), two conditions are typically imposed: (1) \textbf{Signal Stability.} The variance of the pre-activations should remain on the order of a constant, i.e., $\operatorname{Var}(z_i^{(l)}) \sim \mathcal{O}(1)$. (2) \textbf{Initialization Scaling.} Following standard initialization schemes (e.g., He or Xavier initialization), the variance of the weights must be inversely proportional to the input dimension (fan-in): $v_\Theta^{(l)} \propto \frac{1}{w^{(l-1)}}$. Substituting these conditions into the variance equation yields:
        \begin{align}
            \frac{1}{w^{(l-1)}} \cdot \|a^{(l-1)}\|_2^2 \sim \mathcal{O}(1) .
        \end{align}
        Solving for the squared norm, we obtain $\|a^{(l-1)}\|_2^2 \propto w^{(l-1)}$. Taking the square root of both sides leads to the final conclusion:
        \begin{align*}
            \|a^{(l-1)}\|_2 \propto \sqrt{w^{(l-1)}}.
        \end{align*}
    \end{proof}
    \begin{remark}
        While the derivation explicitly relies on the independence assumptions valid at initialization, the scaling law $\|a^{(l-1)}\|_2 \propto \sqrt{w^{(l-1)}}$ holds throughout training in practical settings. This is ensured either by the use of normalization layers (e.g., LayerNorm), which strictly enforce feature variance, or implicitly by the requirement of \textbf{stable signal propagation} in deep networks.
    \end{remark}
\end{lemma}

\begin{proposition}[Gradient Scaling for Weight Matrix]
    \label{app_prop:gradient_scaling_for_weight_matrix}
    The spectral norm of the gradient with respect to the entire weight matrix $\Theta^{(l)}$ scales as:
    \begin{align}
        \|\nabla_{\Theta^{(l)}} f_i(\boldsymbol{x})\|_2 \propto \sqrt{w^{(l-1)}}.
    \end{align}
    \begin{proof}
        We treat the output coordinate $f_i(\boldsymbol{x})$ as the objective function. The backpropagation dynamics remain identical to the standard loss formulation, with the distinction lying solely in the initialization of the backward signal at the final layer.
        By the chain rule, the gradient with respect to the weight matrix is given by the outer product: 
        \begin{align}
            \nabla_{\Theta^{(l)}} f_i = \delta^{(l|i)} (a^{(l-1)})^\top,
        \end{align}
        where $\delta^{(l|i)}$ denotes the error signal at layer $l$ backpropagated specifically from the output node $z_i^{(L)}$, which still satisfies Proposition~\ref{app_prop:stability_of_error_signal}.

        This is a rank-1 matrix. For any rank-1 matrix formed by the outer product of two vectors $\mathbf{u}$ and $\mathbf{v}$, the spectral norm satisfies $\|\mathbf{u}\mathbf{v}^\top\|_2 = \|\mathbf{u}\|_2 \cdot \|\mathbf{v}\|_2$. Applying this property yields:
        \begin{align}
            \|\nabla_{\Theta^{(l)}} f_i\|_2 = \|\delta^{(l|i)}\|_2 \cdot \|a^{(l-1)}\|_2.
        \end{align}

        \textbf{(1) The Error Term (Proposition~\ref{app_prop:stability_of_error_signal}).} The total energy of the backpropagated error is conserved, i.e., $\|\delta^{(l|i)}\|_2 \sim 1$.

        \textbf{(2) The Activation Term (Lemma~\ref{app_lemma:activation_scaling_law}).} The activation norm scales with the width's square root: $\|a^{(l-1)}\|_2 \propto \sqrt{w^{(l-1)}}$.

        Combining these results, the norm of the gradient scales as:
        \begin{align*}
            \|\nabla_{\Theta^{(l)}} f_i\|_2 \propto 1 \cdot \sqrt{w^{(l-1)}} = \sqrt{w^{(l-1)}}.
        \end{align*}
    \end{proof}
\end{proposition}

\begin{proposition}[Signal-Noise Decomposition under High Dimensionality]
    \label{app_prop:signal_noise}
    Under the Signal-Noise Decomposition hypothesis, the cosine similarity between output sensitivity $\boldsymbol{g}_{old} \triangleq \nabla_{\Theta_{(i, S)}} f_{i}(\boldsymbol{x})$ and update direction $\boldsymbol{g}_{new}(s) \triangleq \nabla_{\Theta_{(i+1, r_k)}} \mathcal{L}_{i+1}$ at layer $l$ can be decomposed into a deterministic task-alignment term and a stochastic high-dimensional noise term. With probability at least $1-\delta$, the cosine similarity satisfies the following scaling law:
    \begin{align}
         \cos \theta^{(l)} \propto \left(|\rho_{\text{task}}^{(l)}| 
            +
            \sqrt{\frac{\log(L/\delta)}{D_{\text{eff}}^{(l)}}} \right),
    \end{align}
    where $D_{\text{eff}}^{(l)}$ denotes the effective dimension of the parameter space at layer $l$, and $\rho_{\text{task}}^{(l)}$ represents the semantic conflict component.
    \begin{proof}
        We analyze the gradient interaction in the high-dimensional parameter space. Inspired by the prior work to analysis~\citep{gur2018gradient}, we model the new task gradient $\boldsymbol{g}_{\text{new}}$using the Signal-Noise Decomposition:
        \begin{align}
            \boldsymbol{g}_{\text{new}} = \alpha \boldsymbol{u}_{\text{signal}} + \beta \boldsymbol{v}_{\text{noise}},
        \end{align}
        where $\boldsymbol{u}_{\text{signal}}$ resides in a low-dimensional subspace capturing semantic task conflicts, and $\boldsymbol{v}_{\text{noise}}$ represents the residual component, which is assumed to uniformly distributed on the unit sphere $\mathbb{S}^{D_{\text{eff}}^{(l)}-1}$ of the residual effective subspace based on the \textbf{Conditional Isotropy}.

        Assuming without loss of generality that $\|\boldsymbol{g}_{\text{old}}\|_2=1$, the cosine similarity expands as:
        \begin{align}
            \label{eq:singal_noise_cosine_similarity}
            \cos \theta^{(l)} = \boldsymbol{g}_{\text{old}}^\top \boldsymbol{g}_{\text{new}} 
            = \underbrace{\alpha (\boldsymbol{g}_{\text{old}}^\top \boldsymbol{u}_{\text{signal}})}_{\text{Deterministic: } \rho^{(l)}_{\text{task}}} + \underbrace{\beta (\boldsymbol{g}_{\text{old}}^\top \boldsymbol{v}_{\text{noise}})}_{\text{Stochastic: } \epsilon^{(l)}_{s}}.
        \end{align}

        To bound the stochastic term $\epsilon_{s}$, we employ \textbf{Levy's Lemma (Concentration of Measure on the Sphere)}. Consider the projection function $F(\boldsymbol{v}) = \boldsymbol{g}_{\text{old}}^\top \boldsymbol{v}$. Since $F$ is $1$-Lipschitz and $\boldsymbol{v}_{\text{noise}}$ is isotropically distributed, for any $\epsilon > 0$:
        \begin{align}
            \label{eq:levy's_lemma_application}
            P\left( |\boldsymbol{g}_{\text{old}}^\top \boldsymbol{v}_{\text{noise}}| \geq \epsilon \right) \leq 2 \exp\left( -\frac{D_{\text{eff}}^{(l)} \epsilon^2}{2} \right).
        \end{align}

        We aim to establish a uniform upper bound, denoted as $\epsilon_{\text{orth}}$, such that the gradient orthogonality condition holds across all $L$ layers simultaneously with high probability (e.g., $1 - \delta$).
        
        The probability of the entire network satisfying this condition (``Global Success") can be lower-bounded by the complement of the union of failure events at individual layers. By applying the \textbf{Union Bound}, we derive:
        \begin{align}
            P(\text{Global Success}) 
            &= 1 - P\left(\bigcup_{l=1}^L \{ \text{Failure at layer } l \}\right) \\
            &\geq 1 - \sum_{l=1}^L P(\text{Failure at layer } l) \\
            &= 1 - L \cdot p_{\text{fail}},
        \end{align}
        where $p_{\text{fail}} = P(|\boldsymbol{g}_{\text{old}}^\top \boldsymbol{v}_{\text{noise}}| \geq \epsilon_{\text{orth}})$ represents the failure probability for a single layer.

        We constrain the total failure probability to a small constant $\delta$. Setting $L \cdot p_{\text{fail}} = \delta \to 0$, and substituting Eq,~\ref{eq:levy's_lemma_application}, we have:
        \begin{align}
            p_{\text{fail}} = 2 \exp\left( -\frac{D_{\text{eff}}^{(l)} \epsilon_{\text{orth}}^2}{2} \right) = \frac{\delta}{L}.
        \end{align}
        Solving this equation for $\epsilon_{\text{orth}}$ yields the uniform bound:
        \begin{align}
            \epsilon_{\text{orth}} = \sqrt{\frac{2\log \left( \frac{2L}{\delta} \right)}{D_{\text{eff}}^{(l)}}}.
        \end{align}
        Consequently, for any layer $l \in \{1, \dots, L\}$, the stochastic projection noise $|\epsilon_{s}|$ is bounded with probability $1-\delta$. The noise term scales as:
        \begin{align}
            |\epsilon_{s}| \leq \sqrt{\frac{2\log \left( \frac{2L}{\delta} \right)}{D_{\text{eff}}^{(l)}}} 
            \propto \sqrt{\frac{\log(L/\delta)}{D_{\text{eff}}^{(l)}}}.
        \end{align}
        Substituting this scaling behavior back into the cosine decomposition (Eq.~\eqref{eq:singal_noise_cosine_similarity}), we obtain:
        \begin{align*}
            \cos \theta^{(l)} \propto \left(|\rho_{\text{task}}^{(l)}| 
            +
            \sqrt{\frac{\log(L/\delta)}{D_{\text{eff}}^{(l)}}} \right).
        \end{align*}
    \end{proof}
\end{proposition}

\begin{lemma}[Asymptotic Vanishing of the Effective Dimension]
    \label{app_lemma:asymptotic_vanishing} 
    Let the gradient space at layer $l$ admit the orthogonal decomposition $\mathbb{R}^{N^{(l)}} = V_{\mathrm{signal}} \oplus V_{\mathrm{noise}}$, where $N^{(l)} = w^{(l)} w^{(l-1)}$ denotes the nominal parameter dimension. Suppose the signal subspace dimension $\dim(V_{\mathrm{signal}}) = k$ is independent of model width. Then in the wide-network limit $w^{(l)} \to \infty$,
    \begin{align}
        \lim_{w^{(l)} \to \infty} \sqrt{\frac{\log(L/\delta)}{D_{\text{eff}}^{(l)}}} = 0.
    \end{align}
    \begin{proof}
        By \citet{gur2018gradient}, gradient updates during training reside predominantly in a low-dimensional subspace spanned by the top eigenvectors of the Hessian, whose dimension equals the number of output classes $k$ — a quantity determined solely by the task and invariant to model width. Since $V_{\mathrm{noise}} = V_{\mathrm{signal}}^{\perp}$, the effective noise dimension satisfies
        
        $$D_{\mathrm{eff}}^{(l)} = N^{(l)} - k = w^{(l)}w^{(l-1)} - k.$$
        
        For fixed $k$ and $w^{(l-1)}$, we have $D_{\mathrm{eff}}^{(l)} \to \infty$ as $w^{(l)} \to \infty$. Since $\log(L/\delta)$ is independent of $w^{(l)}$, it follows immediately that
        
        $$\lim_{w^{(l)} \to \infty} \sqrt{\frac{\log(L/\delta)}{D_{\mathrm{eff}}^{(l)}}} = \lim_{w^{(l)} \to \infty} \sqrt{\frac{\log(L/\delta)}{w^{(l)}w^{(l-1)} - k}} = 0.$$
    \end{proof}
\end{lemma}

\begin{theorem}[Scaling Law of Output Shift]
\label{thm:scaling_law_of_shift}
Let $f_i(\boldsymbol{x})$ and $f_{i+1}(\boldsymbol{x})$ denote the model output before and after learning task $i+1$, respectively. Under the assumption that second-order terms are negligible (dominating $<20\%$ of the variation), the shift in the output function for a previous task sample $x$ scales as:
\begin{align}
    \| f_{i+1}(\boldsymbol{x})-f_{i}(\boldsymbol{x}) \|_2 
    &\propto 
    \sum_{l=1}^L (w^{(l-1)})^{\alpha+\frac{1}{2}} (w^{(l)})^{\beta} \cdot |l^b e^{-cl}|,
    \end{align}
\end{theorem}

\begin{proof}
    We analyze the output shift using a first-order Taylor approximation around the parameters $\Theta_i$. Following the local linearization perspective widely adopted in deep learning theory analyses~\citet{jacot2018neural, lee2019wide}, we have:
    \begin{align}
        \| f_{i+1}(\boldsymbol{x}) - f_{i}(\boldsymbol{x}) \|_2 \approx \| \nabla_{\Theta} f_i(\boldsymbol{x})^\top \Delta \Theta_i \|_2.
    \end{align}
    Substituting the gradient flow dynamics $\Delta \Theta_i = \int_0^S \nabla_{\Theta} \mathcal{L}_{i+1}(\Theta_s) ds$, and applying the triangle inequality for integrals, we obtain:
    \begin{align}
        \| f_{i+1}(\boldsymbol{x}) - f_{i}(\boldsymbol{x}) \|_2 
        &= \left\| \int_0^S \langle \nabla_{\Theta} f_i(\boldsymbol{x}), \nabla_{\Theta} \mathcal{L}_{i+1}(s) \rangle ds \right\|_2 \\
        &\leq \int_0^S \left| \sum_{l=1}^L \langle \boldsymbol{g}_{\text{old}}^{(l)}, \boldsymbol{g}_{\text{new}}^{(l)}(s) \rangle \right| ds \label{app_eq:gradient_alignment_dynamics}\\
        &\leq \sum_{l=1}^L \underbrace{\|\boldsymbol{g}_{\text{old}}^{(l)}\|_2 \int_0^S \|\boldsymbol{g}_{\text{new}}^{(l)}(s)\|_2 \cdot |\cos \theta_{s}^{(l)}| ds}_{\text{Term}^{(l)}},
    \end{align}
    where $\theta_{s}^{(l)}$ is the angle between the fixed old task gradient $\boldsymbol{g}_{\text{old}}^{(l)}$ and the time-varying new task gradient $\boldsymbol{g}_{\text{new}}^{(l)}(s)$. 
    
    We analyze the bound layer by layer:
    \begin{enumerate}
        \item \textbf{Considering $\boldsymbol{g}_{\text{old}}^{(l)}$.} From Proposition~\ref{prop:gradient_scaling_for_weight_matrix}, the gradient norm scales as $\|\boldsymbol{g}_{\text{old}}^{(l)}\|_2 = \|\nabla_{\Theta^{(l)}} f_i(\boldsymbol{x})\|_2 \propto \sqrt{w^{(l-1)}}$.
        
        \item \textbf{Considering $\int_0^S \|\boldsymbol{g}_{\text{new}}^{(l)}(s)\|_2 \cdot |\cos \theta_s^{(l)}| ds$.} 
        By Assumption~\ref{trajectory regularity}, the optimization trajectory exhibits high directional stability, implying that the angle $\theta_s^{(l)}$ remains nearly constant throughout the update. We thus approximate the instantaneous alignment by its characteristic value $|\cos \theta^{(l)}|$. 
        \begin{align}
            \int_0^S \|\boldsymbol{g}_{\text{new}}^{(l)}(s)\|_2 \cdot |\cos \theta_s^{(l)}| ds
            &\lesssim
            \left( \int_0^S \|\boldsymbol{g}_{\text{new}}^{(l)}(s)\|_2 ds \right) \cdot |\cos \theta^{(l)}| 
            = \mathrm{PathLen}^{(l)} \cdot |\cos \theta^{(l)}|
        \end{align}
        
        \item \textbf{Combination.} Incorporating Lemma~\ref{lemma:pathlen_scaling_law} ($\mathrm{PathLen}^{(l)} \propto w^{(l-1)\,\alpha}{w^{(l)\, \beta}}$) and~\eqref{eq:theorem_depth_term} ($\cos \theta \propto l^b e^{-cl}$),  the total contribution of layer $l$ satisfies:
        \begin{align}
            \text{Term}^{(l)} \lesssim \sqrt{w^{(l-1)}} \cdot w^{(l-1)\,\alpha}\cdot {w^{(l)\, \beta}} \cdot |l^b e^{-cl}| = (w^{(l-1)})^{\alpha+\frac{1}{2}} (w^{(l)})^{\beta} \cdot |l^b e^{-cl}|.
        \end{align}
    \end{enumerate}
    Summing over all layers $l=1, \dots, L$ yields the final result.
\end{proof}

\section{Detailed Validation of Theoretical Framework}
\subsection{Experimental Setup}
\paragraph{Architectural Diversity.} Unlike prior studies that typically assume uniform widths, we constructed a comprehensive pool of Fully-Connected Neural Networks (FNNs) with highly heterogeneous topologies. The network depth was varied ($L \in \{3, 5, 10, 15, 20, 25\}$), and hidden layer widths were sampled from a candidate pool $\mathcal{W} \in \{256, 512, 1024, 2048, 4096\}$. To ensure structural generality, we generated five distinct topological categories: (1) Uniform: Constant width across all layers. (2) Monotonic: strictly Increasing or Decreasing widths. (3) Non-Monotonic: Bottleneck (compression followed by expansion), Spindle (expansion followed by compression), and fully Randomized configurations. This resulted in a total of 175 unique architecture instances per task scenario.

What's more, our experiments further evaluate on convolutional neural networks (CNNs) and vision transformers (ViTs). Specifically, we adopt VGG family (VGG-13, VGG-16 and VGG-19) for CNNs. For ViTs, we adjust the dimension of matrix $Q, K, V$ and the number of Transformer blocks to obtain variable ``width" and ``depth" (width: [6, 8, 12, 16]; depth: [128, 192, 256, 384]). To enable comprehensive statistical analysis across diverse architectures, we additionally vary the expansion ratio of the feed‑forward network (FFN) as $\text{ratio}=\frac{d_{\textrm{mid}}}{d_{\textrm{in}}}$, where $d_{\textrm{mid}}$ is the hidden dimension of the FFN and $d_{\textrm{in}}$ is its input dimension, which is equal to the dimension of $Q, K, V$.

\paragraph{Training Protocol.} We performed supervised learning on standard benchmark datasets (e.g., MNIST, FashionMNIST). The networks were initialized using Kaiming Normal initialization for ReLU activation. Training was conducted using the SGD optimizer with a learning rate of $1 \times 10^{-3}$, momentum of $0.9$, and weight decay of $5 \times 10^{-4}$. We tracked the weight matrices before ($\Theta^{(l)}_{0}$) and after ($\Theta^{(l)}_{S}$) the training phase.

\subsection{Additional Evidence of Width Scaling Assumption}
\label{app_sec:more_evidence_of_ass1}
Beyond visual evidence of Assumption~\ref{ass:width_term}, we performed a log-log linear regression $\log\!\left(\frac{\|\Delta\Theta\|}{\|\Theta\|}\right)= c + \alpha \log(w_{\text{in}}) + \beta \log(w_{\text{out}})$ across thousands of layer-wise configurations sampled from networks of varying depths (from 3 to 25 layers). The quantitative results are summarized in Table~\ref{tab:width_scaling_verification}.

\begin{table}[t]
  \vspace{\baselineskip} % one line space before the table title
  \caption{Statistical Verification of the Width Scaling Law (Assumption 4.1).}
  \label{tab:width_scaling_verification}
  \vspace{\baselineskip} % one line space after the table title
  \begin{center}
    \begin{small}
      \begin{tabular}{c c c c c c c}
        \toprule
        \multirow{2}{*}{\shortstack{\textbf{Network}\\\textbf{Depth}}}
        & \multirow{2}{*}{\shortstack{\textbf{Sample}\\ \textbf{Size}\textbf{$(n)$}}}
        & \multirow{2}{*}{\shortstack{$\boldsymbol{\alpha}_{\pm \textbf{SE}}$\\\textbf{(for $w_{\text{in}}$)}}}
        & \multirow{2}{*}{\shortstack{$\boldsymbol{\beta}_{\pm \textbf{SE}}$\\\textbf{(for $w_{\text{out}}$)}}}
        & \multirow{2}{*}{\shortstack{\textbf{Pearson $r$}}}
        & \multirow{2}{*}{\shortstack{\textbf{Spearman $r_s$}}}
        & \multirow{2}{*}{\shortstack{\textbf{$p$-value}}} \\
        \multicolumn{1}{c}{} & & & & & & \\
        \midrule
        3  & 50  & $0.214_{\pm 0.023}$ & $-0.486_{\pm 0.021}$ & 0.962 & 0.961 & $< 0.001$ \\
        5  & 120 & $0.201_{\pm 0.017}$ & $-0.426_{\pm 0.017}$ & 0.917 & 0.897 & $< 0.001$ \\
        10 & 270 & $0.146_{\pm 0.016}$ & $-0.351_{\pm 0.016}$ & 0.818 & 0.858 & $< 0.001$ \\
        15 & 420 & $0.109_{\pm 0.014}$ & $-0.286_{\pm 0.014}$ & 0.752 & 0.813 & $< 0.001$ \\
        20 & 570 & $0.124_{\pm 0.013}$ & $-0.299_{\pm 0.013}$ & 0.740 & 0.812 & $< 0.001$ \\
        25 & 720 & $0.101_{\pm 0.012}$ & $-0.279_{\pm 0.012}$ & 0.726 & 0.801 & $< 0.001$ \\
        \bottomrule
      \end{tabular}
    \end{small}
  \end{center}
  \vspace{\baselineskip}
  \emph{Note:} The symbols $c$, $\alpha$, and $\beta$ in the regression equation above are introduced solely for this statistical analysis and are \textbf{NOT} related to any variables of the same name appearing in the main paper. To avoid ambiguity, we use these as generic regression coefficients (intercept and slopes) only within the scope of this validation.
\end{table}

\subsection{Additional Evidence of Middle-Layer Vulnerability}
As demonstrated in Figure~\ref{fig:more_support_for_middle_layer_vulnerability}, the middle-layer vulnerability exist across scenarios (different combination of various datasets) and model families (fully-connected neural networks and convolutional neural networks). 

\label{app_sec:more_evidence_of_ass2}

\begin{figure}[htbp]
    \centering

    \begin{subfigure}{\linewidth}
        \centering
        \includegraphics[width=\linewidth]{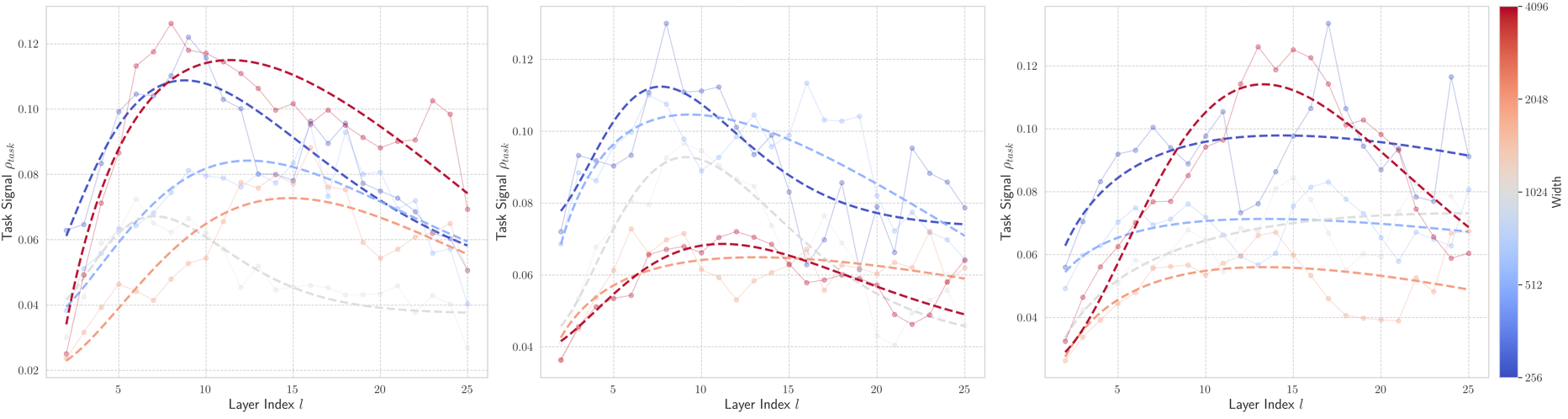}
        \caption{Validation on fully-connected neural network (FNN). Specifically, the subfigures from left to right depict the middle-layer vulnerability of FNNs under the transfer scenarios MNIST $\to$ CIFAR-10, FashionMNIST $\to$ MNIST, and FashionMNIST $\to$ CIFAR-10, respectively.}
        \label{fig:middle_layer_vulnerability_mlp}
    \end{subfigure}

    \begin{subfigure}{\linewidth}
        \centering
        \includegraphics[width=\linewidth]{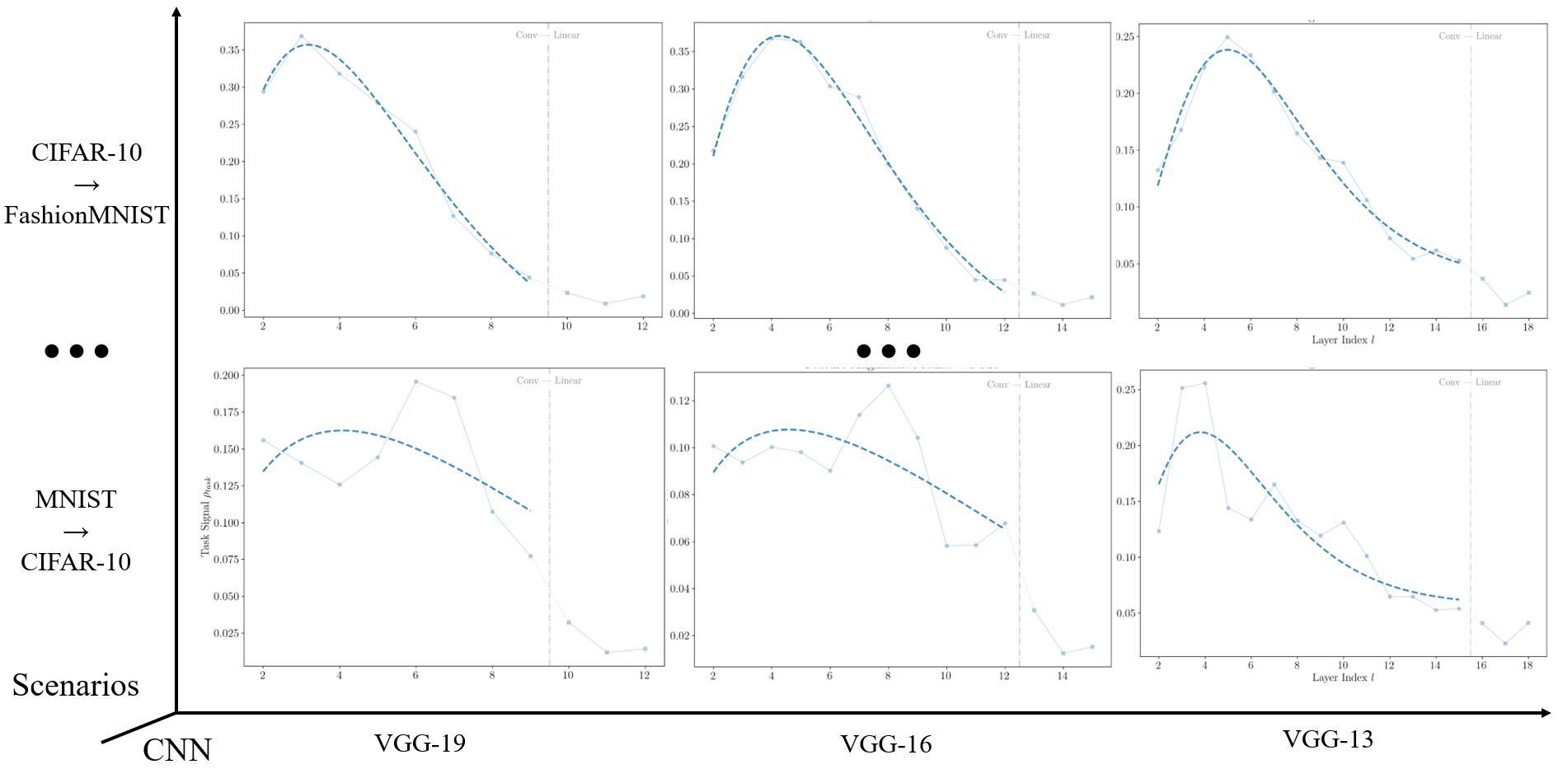}
        \caption{Validation on convolutional neural network (CNN).}
        \label{fig:middle_layer_vulnerability_cnn}
    \end{subfigure}
    
    \caption{Empirical Validation of the Middle-Layer Vulnerability across several scenarios: calibration efficacy (left), filtering trade-offs via precision/recall curves (middle), and PR curve with AUC-PR (right). Each row corresponds to a different continual learning scenario.}
    \label{fig:more_support_for_middle_layer_vulnerability}
\end{figure}

\subsection{Strong correlation between ADS and logit shift in Transformer}
We validate the scope of ADS by comparing ADS and theory proposed by~\citeauthor{guha2024diminishing} in Vision Transformer (ViT). Considering the huge quantity requirement of ViT, we choose CIFAR-100 and ImageNet as our benchmarks, and we resize ImageNet's data to $32\times 32$ to align with CIFAR-100.
\begin{figure}[h]
\centering
\includegraphics[width=\linewidth]{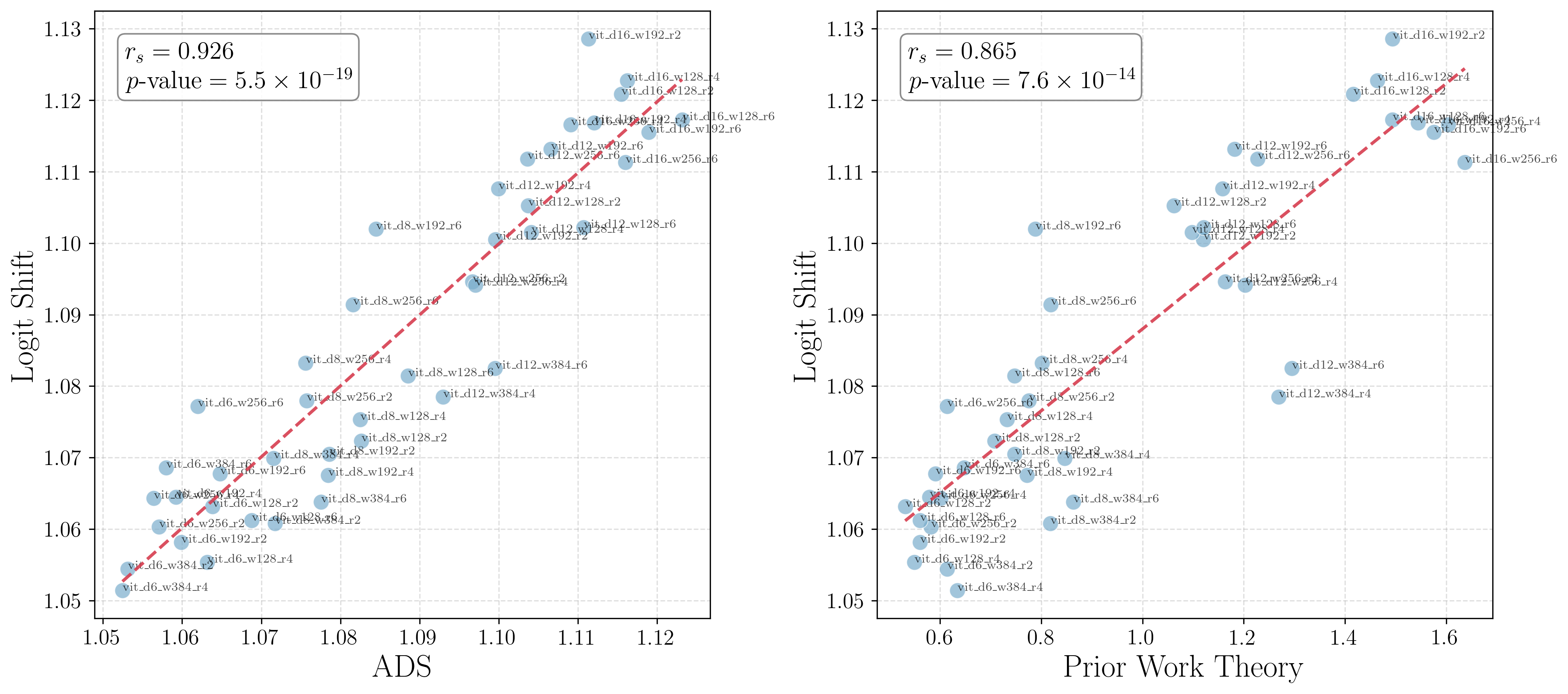}
\caption{Correlation comparison between empirically observed logit shift and theoretical metrics across heterogeneous model architectures on CIFAR-10 to ImageNet scenario. The scatter plots illustrate the relationship between the logit shift and two theoretical predictors evaluated on various ViT configurations. The dashed red line represents the trend line.}
\label{app_fig:vit_theory_logit_shift}
\end{figure}
As shown in Figure~\ref{app_fig:vit_theory_logit_shift}, the proposed ADS metric exhibits a stronger Spearman rank correlation ($r_s = 0.926$) with the logit shift compared to the theory from prior work ($r_s = 0.865$).
\subsection{Validation of Expected Calibration Error}
\label{app_sec:application}
\subsubsection{Experimental Setup}
We choose 30\% proportion of dataset to calibrate parameter. There are several scenarios to validate the performance of ADS-based selector as shown Figure~\ref{fig:application_ece_full}.

\subsection{Experimental Results}
The empirical validation of the ADS-based selector's performance in reliable continual learning model selection is shown in Figure~\ref{fig:application_ece_full}. Although the selected architectures do not achieve the ideal lowest Expected Calibration Error (ECE), they are more reliable than both the vanilla architecture and those calibrated with a post-hoc temperature scaling mechanism. As illustrated in the middle column of Figure~\ref{fig:application_ece_full}, the ADS-based selector serves as an effective coarse-grained selector. It achieves high precision at low thresholds and high recall at high thresholds, indicating that it can identify the best-performing models or filter out the worst-performing ones with high probability. Furthermore, the high AUC-PR values across various scenarios demonstrate that the ADS-based selector consistently outperforms the random baseline, which typically yields an AUC-PR of $0.51$.

\begin{figure}[htbp]
    \centering

    % (a) MNIST
    \begin{subfigure}{\linewidth}
        \centering
        \includegraphics[width=0.8\linewidth]{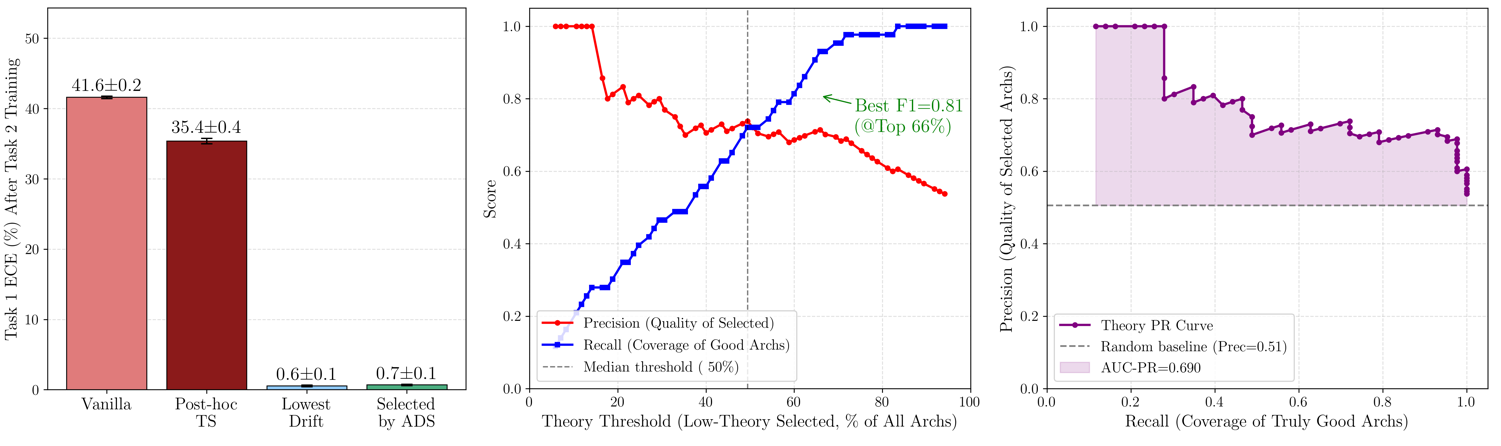}
        \caption{MNIST (classes 5-9 $\to$ classes 0-4)}
        \label{fig:ece_mnist_inv}
    \end{subfigure}

    % \vspace{0.5em}

    % (b) FashionMNIST 0-4 -> 5-9
    \begin{subfigure}{\linewidth}
        \centering
        \includegraphics[width=0.8\linewidth]{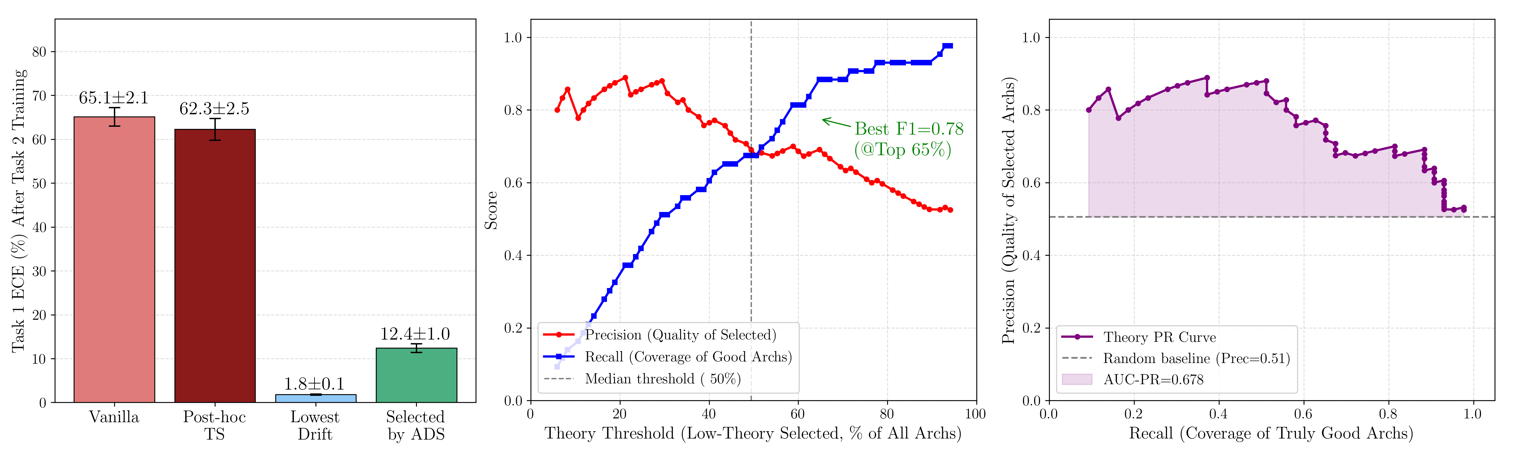}
        \caption{FashionMNIST (classes 0-4 $\to$ classes 5-9)}
        \label{fig:ece_fmnist}
    \end{subfigure}

    % \vspace{0.5em}

    % (c) FashionMNIST 5-9->0-4
    \begin{subfigure}{\linewidth}
        \centering
        \includegraphics[width=0.8\linewidth]{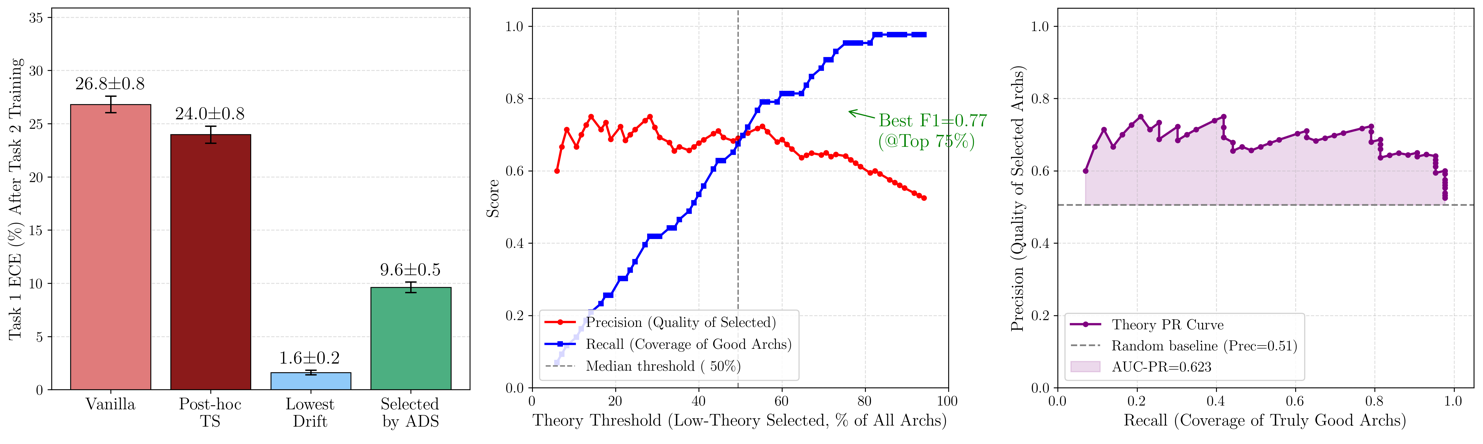}
        \caption{FashionMNIST (classes 5-9 $\to$ classes 0-4)}
        \label{fig:ece_fmnist_inv_cifar}
    \end{subfigure}

    % \vspace{0.5em}
    
    % (d) CIFAR-10 0-4 -> 5-9
    \begin{subfigure}{\linewidth}
        \centering
        \includegraphics[width=0.8\linewidth]{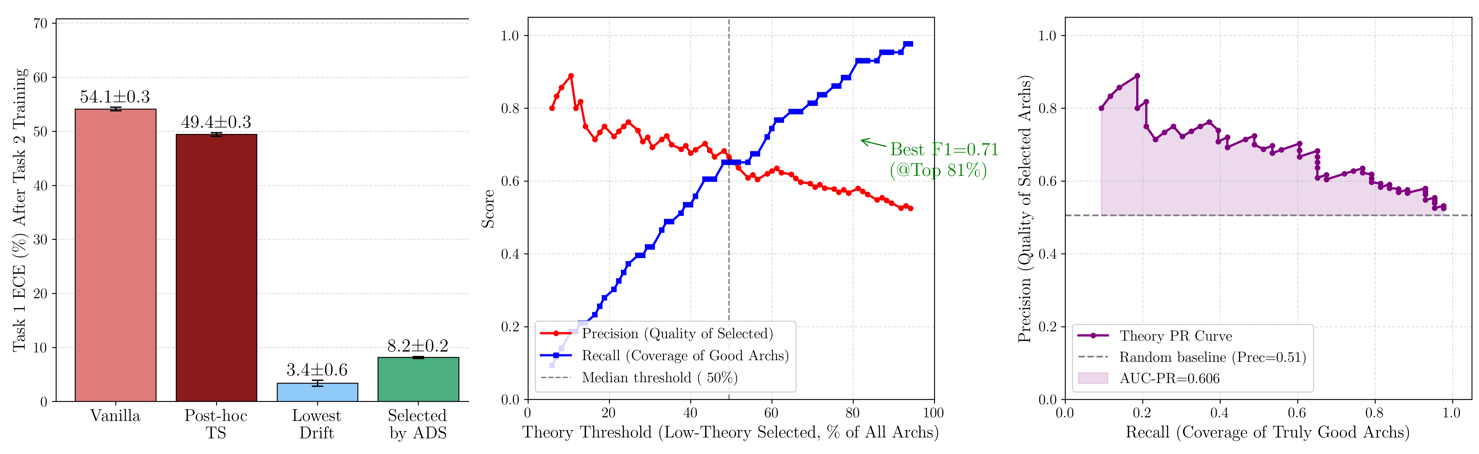}
        \caption{CIFAR-10 (classes 5-9 $\to$ classes 0-4)}
        \label{fig:ece_cifar}
    \end{subfigure}

    % \vspace{0.5em}
    
    % (e) CIFAR-10 5-9->0-4
    \begin{subfigure}{\linewidth}
        \centering
        \includegraphics[width=0.8\linewidth]{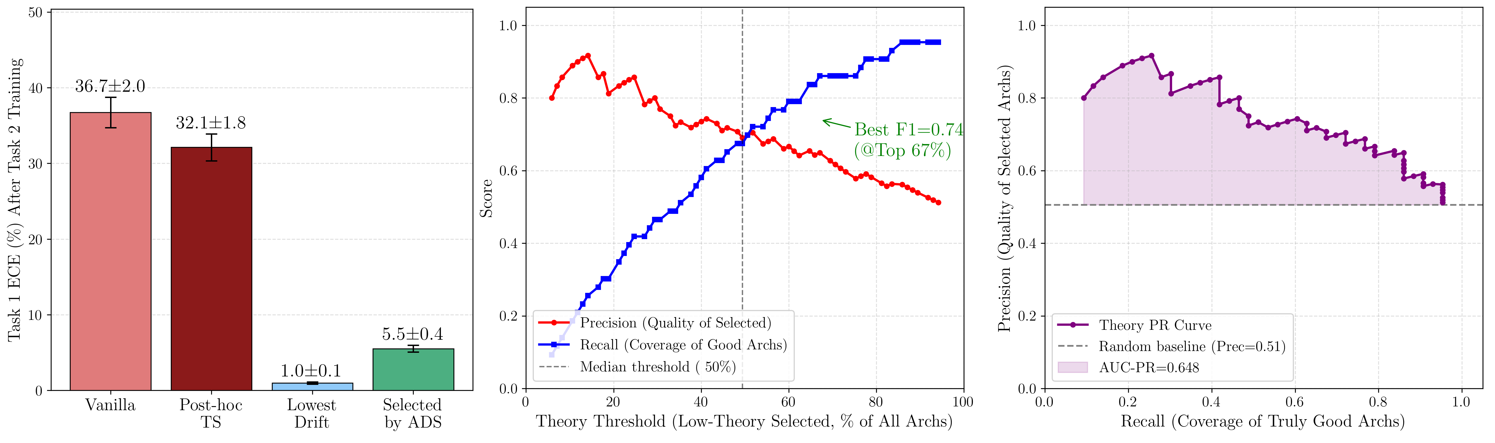}
        \caption{CIFAR-10 (classes 5-9 $\to$ classes 0-4)}
        \label{fig:ece_cifar_inv}
    \end{subfigure}

    \caption{Quantitative evaluation of the ADS-based selector across all benchmarks:
             calibration efficacy (left), filtering trade-offs via precision/recall curves (middle),
             and PR curve with AUC-PR (right). Each row corresponds to a different
             continual learning scenario.}
    \label{fig:application_ece_full}
\end{figure}

\end{document}

%% file: math_commands.tex
\usepackage{amsmath,amsfonts,bm}

\def\eqref#1{equation~\ref{#1}}
\def\1{\bm{1}}

\DeclareMathAlphabet{\mathsfit}{\encodingdefault}{\sfdefault}{m}{sl}
\SetMathAlphabet{\mathsfit}{bold}{\encodingdefault}{\sfdefault}{bx}{n}